\pdfoutput=1

\documentclass[journal]{IEEEtran} 
\usepackage{amsmath}
\usepackage{amsthm}
\usepackage{amsfonts}
\usepackage{amssymb}
\usepackage{amscd}
\usepackage{graphicx}
\usepackage{newlfont}
\usepackage{cite}
\usepackage{color}
\usepackage{multirow,tabularx}
\usepackage{ifthen}
\usepackage{enumitem}
\usepackage{times}
\usepackage{stackrel}
\usepackage[ruled]{algorithm2e}
\usepackage{algpseudocode}
\usepackage{float}
\usepackage[all]{xy}
\usepackage{url}
\usepackage[font = footnotesize]{caption}
\usepackage{subcaption}
\usepackage{booktabs}
\usepackage{multirow,tabularx}
\usepackage{hyphenat}

\def\BibTeX{{\rm B\kern-.05em{\sc i\kern-.025em b}\kern-.08em
    T\kern-.1667em\lower.7ex\hbox{E}\kern-.125emX}}

\IEEEoverridecommandlockouts

\newcommand{\lb}{\left(}
\newcommand{\rb}{\right)}

\newtheorem{theorem}{{Theorem}}
\newtheorem{lemma}{{Lemma}}

\newcommand{\norm}[1]{\Vert#1\Vert}

\newcommand{\diag}{\mathrm{diag}}

\newcommand{\M}{\mathcal{M}}
\newcommand{\N}{\mathcal{N}}
\newcommand{\B}{\mathcal{B}}
\newcommand{\J}{\mathcal{J}}
\newcommand{\Real}{\mathbb{R}}

\newcommand{\ls}{\left[}
\newcommand{\rs}{\right]}
\newcommand{\lc}{\left\{}
\newcommand{\rc}{\right\}}
\newcommand{\veca}{\mathbf{a}}
\newcommand{\vecc}{\mathbf{c}}

\newcommand{\vecx}{\mathbf{x}}
\newcommand{\vecz}{\mathbf{z}}
\newcommand{\vecy}{\mathbf{y}}

\newcommand{\vecu}{\mathbf{u}}
\newcommand{\vecb}{\mathbf{b}}
\newcommand{\vecw}{\mathbf{w}}
\newcommand{\matA}{\mathbf{A}}
\newcommand{\matB}{\mathbf{B}}
\newcommand{\matI}{\mathbf{I}}
\newcommand{\matC}{\mathbf{C}}
\newcommand{\matE}{\mathbf{E}}

\newcommand{\matG}{\mathbf{G}}

\newcommand{\matX}{\mathbf{X}}
\newcommand{\matY}{\mathbf{Y}}
\newcommand{\matPhi}{\mathbf{\Phi}}
\newcommand{\matSigma}{\mathbf{\Sigma}}
\newcommand{\matGamma}{\mathbf{\Gamma}}
\newcommand{\vgamma}{\boldsymbol{\gamma}}
\newcommand{\vmu}{\boldsymbol\mu}

\newcommand{\vlambda}{\boldsymbol{\lambda}}
\newcommand{\vlambdajb}{\vlambda_{j}^{b}}
\newcommand{\vgammab}{\vgamma_{b}}
\newcommand{\vgammaj}{\vgamma_{j}}

\newcommand{\vecxall}{\vecx_{1}, \vecx_{2}, \dots, \vecx_{L}}

\allowdisplaybreaks

\begin{document}

\title{Decentralized Joint-Sparse Signal Recovery: A Sparse Bayesian Learning Approach \thanks{This work has appeared in part in \cite{Khanna14CBDSBL}.}}
\author{\authorblockN{Saurabh Khanna, \emph{Student Member, IEEE} and Chandra R. Murthy, \emph{Senior Member, IEEE}}\\
\authorblockA{\begin{tabular}{cc}
Dept. of ECE, Indian Institute of Science \\
Bangalore, India \\
\{sakhanna,~cmurthy\}@ece.iisc.ernet.in \\
\end{tabular}}
\vspace{-8mm}
}


\maketitle

\begin{abstract}
This work proposes a decentralized, iterative, Bayesian algorithm called CB-DSBL for in-network estimation of multiple jointly sparse vectors by a network of nodes,
using noisy and underdetermined linear measurements. The proposed algorithm exploits the network wide joint sparsity of the unknown sparse vectors to recover them
from significantly fewer number of local measurements compared to standalone sparse signal recovery schemes. To reduce the amount of inter-node communication and the associated overheads, the nodes exchange messages with only a small subset of their single hop neighbors.  Under this communication scheme,
we separately analyze the convergence of the underlying Alternating Directions Method of Multipliers (ADMM) iterations used in our proposed algorithm and establish 
its linear convergence rate. The findings from the convergence analysis of decentralized ADMM are used to accelerate the convergence 
of the proposed CB-DSBL algorithm. Using Monte Carlo simulations, we demonstrate the superior signal reconstruction as well as support recovery performance of 
our proposed algorithm compared to existing decentralized algorithms: DRL-1, DCOMP and DCSP.

\end{abstract}

\section{Introduction}
We consider the problem of in-network estimation of multiple joint-sparse vectors by a network of
 connected agents or processing nodes, using noisy and underdetermined linear measurements. Two or more vectors 
 in $\Real^{n}$ are called \emph{joint-sparse} if, in addition to each vector being individually sparse,\footnote{
 A vector in $\Real^{n}$ is said to be $k$-sparse if only $k (\ll n)$ out of its $n$ coefficients are nonzero.} 
 their nonzero coefficients belong to a common index set.
Joint sparsity occurs naturally in scenarios involving multiple agents trying to learn a sparse representation 
of a common physical phenomenon. Since the underlying physical phenomenon is the same for all the agents 
(with similar acquisition modalities), their individual sparse representations/model parameters tend to exhibit 
joint sparsity. In this work, we consider joint-sparse vectors which belong to Type-2 Joint Sparse Model \cite{Duarte05DCS} or JSM-2, 
one of the three generative models for joint-sparse signals. JSM-2 signal vectors satisfy the property that their 
nonzero coefficients are uncorrelated within and across the vectors. 
JSM-2 has been successfully used in several applications such as cooperative spectrum sensing 
\cite{Makhzani13DCSAMP, Tian11DistSpectrumSensing}, decentralized event detection \cite{Ling_13_jsm_lqnorm, Nasrabadi11MTMV},
multi-task compressive sensing \cite{Shihao09MultiTaskCS} and MIMO channel estimation\cite{Ranjitha15SBLChEst, Masood15SparseMIMOChEst, Barbotin12MIMOJSM}.

To further motivate the signal structure of joint sparsity in a distributed setup, consider the problem of detection/classification of randomly occurring events
in a field by multiple sensor nodes. Each sensor node $j$, $1 \le j \le L$, employs a dictionary $\Psi_{j} = [\psi_{j}^{1}; \psi_{j}^{2} \dots \psi_{j}^{c}]$,
whose each column $\psi_{j}^{i}$ is the signature corresponding to the $i^{\text{th}}$ event, one out of the $c$ events which can potentially occur. 
In many cases, due to the inability to accurately model the sensing process, the signature vectors $\psi_{j}^{i}$ are simply chosen to be the past recordings
of $j^{\text{th}}$ sensor corresponding to standalone occurrence of the $i^{\text{th}}$ event, averaged across multiple experiments \cite{Nasrabadi11MTMV}. This 
procedure can result in a dictionary whose columns are highly correlated. Thus, for any $k$ $(\ll c)$ events occurring 
simultaneously, a noisy sensor recording might belong to multiple subspaces, each spanned by different subsets of 
columns of the local dictionary. In such a scenario, enforcing joint sparsity across the sensor nodes can resolve 
the ambiguity in selecting the correct subset of columns at each sensor node.

In this work, we consider a distributed setup where each individual joint-sparse vector is estimated
by a distinct node in a network comprising multiple nodes, with each node having access to noisy and underdetermined
linear measurements of its local sparse vector. By collaborating with each other, these nodes can exploit the underlying joint sparsity of their local sparse vectors to reduce the measurements required per node or improve the quality of their local signal estimates.
In \cite{Duarte05DCS}, it has been shown that the number of local measurements required for common support recovery can be dramatically
reduced by exploiting the joint sparsity structure prevalent across the network. In fact, as the nodes increase in number,
exact signal reconstruction is possible from as few as $k$ measurements per node, where $k$ denotes the size of the
support set. Such a substantial reduction in the number of measurements is highly desirable, especially in applications where
the cost or time required to acquire new measurements is high. 

Distributed algorithms for JSM-2 signal recovery come in two flavors - centralized and decentralized. In the centralized approach, each
node transmits its local measurements to a fusion center (FC) which runs a joint-sparse signal recovery algorithm. The FC then transmits the reconstructed
sparse signal estimates back to their respective nodes. In contrast, in a decentralized approach, the goal is to obtain the same solution as with the centralized scheme
at all nodes by allowing each node to exchange information with its single hop neighbors in addition to processing its local measurements. Besides being inherently
robust to node failures, decentralized schemes also tend to be more energy efficient as the inter-node communication is restricted to relatively short ranges
covering only one hop communication links. In this work, we focus on the decentralized approach for solving the sparse signal recovery problem under
the JSM-2 signal model. 

\subsection{Related Work} \label{sec:related_work}
In this subsection, we briefly summarize the existing centralized and decentralized algorithms for JSM-2 signal recovery. The earliest work on joint-sparse signal 
recovery considered extensions of recovery algorithms meant for single measurement vector setup to the centralized multiple measurement vector (MMV) model \cite{Cotter_05_mmv}, and demonstrated the significant performance gains that are achievable by exploiting the joint sparsity structure. MMV Basic Matching Pursuit (M-BMP), MMV Orthogonal Matching Pursuit (M-OMP) and MMV FOcal Underdetermined System Solver (M-FOCUSS), introduced in \cite{Cotter_05_mmv}, belong to this category.
In \cite{Duarte09DCS_main}, joint sparsity was exploited
for distributed encoding of multiple sparse signals. This work generalized the joint-sparse signals as being generated according to one of the three joint-sparse
signal models (JSM-1,2,3). This work also proposed a centralized greedy algorithm called Simultaneous Orthogonal Matching Pursuit (SOMP) \cite{Duarte05DCS} for
JSM-2 recovery. In \cite{Lu11AdmmMmv}, Alternating Directions Method for MMV setup (ADM-MMV) was proposed which used an $\ell_{2} / \ell_{1}$ mixed norm penalty 
to promote a joint-sparse solution. 
In \cite{Wipf_07_msbl}, the multiple response sparse Bayesian learning (M-SBL) algorithm was proposed as an MMV extension of the SBL algorithm \cite{Wipf_04_sbl}.
Unlike the algorithms discussed earlier, M-SBL adopts a probabilistic approach by seeking the maximum a posterior probability (MAP) estimate of the JSM-2 signals.
In M-SBL, a joint-sparse solution is encouraged by assuming a joint sparsity inducing parameterized prior on the unknown sparse vectors, with the prior parameters learnt directly from the measurements. 
M-SBL has been shown to outperform deterministic methods based on $\ell_{0}$ norm relaxation such as M-BMP and M-FOCUSS \cite{Cotter_05_mmv} as well as greedy algorithms such as SOMP. 
AMP-MMV \cite{Ziniel13AMPMMV} is another Bayesian algorithm which uses approximate message passing (AMP) to obtain marginalized conditional posterior
distributions of joint-sparse signals. Owing to their low computational complexity, AMP based algorithms are suitable for recovering signals with large dimensions.
However, they have been shown to converge only for large dimensional and randomly constructed measurement matrices. Interested readers are referred to
\cite{Rakotomamonjy_survey} for an excellent study comparing some of the aforementioned centralized JSM-2 signal recovery algorithms.

Among decentralized algorithms, collaborative orthogonal matching pursuit (DCOMP) \cite{Wimalajeewa13DCOMP} and 
collaborative subspace pursuit (DCSP) \cite{Varshney14DCSP} are greedy algorithms for JSM2 signal recovery, and both are computationally very fast. 
However, as demonstrated later in this paper, they do not perform as well as regularization based methods which induce joint sparsity in
their solution by employing a suitable penalty or indirectly via a joint signal prior. Moreover, both DCOMP and DCSP assume a priori knowledge of the size of the nonzero support set, which could be unknown
or hard to estimate. Decentralized row-based LASSO (DR-LASSO) \cite{LingT11DecentralizedSuppportDetect} is an iterative alternating minimization algorithm which
optimizes a non-convex objective with $\ell_{1}$-$\ell_{2}$ mixed norm based regularization to obtain a joint-sparse solution. 
Decentralized re-weighted $\ell_{1}(\ell_{2})$ minimization algorithms DRL-1,2 \cite{Ling_13_jsm_lqnorm} employ a non-convex sum-log-sum penalty to promote
a joint-sparse solution. Although non-convex regularizers induce sparsity much more strongly
as compared to convex $\ell_{1}$ norm based regularizers \cite{Bach12SparsityInducingPenalties}, the resulting non-convex optimization 
can be difficult to solve efficiently. In DRL-1/2, the non-convex objective is replaced by a surrogate convex function constructed from iteration dependent weighted
$\ell_{1}$/$\ell_{2}$ norm terms. Using a non-convex sum-log-sum regularization results in a more sparse solution compared to convex regularization used
in DR-LASSO. However, both DR-LASSO and DRL-1,2 necessitate cross validation to tune the amount of regularization needed for optimal support recovery performance.
DRL-1,2 also requires proper tuning of a so-called smoothing parameter and an ADMM parameter for its optimal performance. By employing a Bayesian approach,we can completely eliminate any need for cross validation, by learning the parameters of a family of signal priors, such that selected signal prior has maximum Bayesian
evidence. DCS-AMP \cite{Makhzani13DCSAMP} is one such decentralized algorithm which employs approximate message passing to learn a parameterized joint sparsity inducing
Bernoulli-Gaussian signal prior. Turbo Bayesian Compressive Sensing (Turbo-BCS) \cite{Yang10TurboBCS}, another decentralized algorithm, adopts a more relaxed zero mean
Gaussian signal prior, with the variance hyperparameters themselves distributed according to an exponential distribution. 
This relaxation of signal prior results in improved MSE without compromising on sparsity of the solution. Turbo-BCS, however, involves direct
exchange of signal estimates between the nodes, which renders it unsuitable for applications where it is necessary to preserve the privacy of the local signals.

\begin{table*}[bt]
\scriptsize
\begin{center}
\caption{Comparison of Decentralized Joint-Sparse Signal Recovery Algorithms}
\label{tab:table_jss_algo_compare}
\begin{tabular*}{0.85\textwidth}{@{}l | c | c | c | c | c@{}}
\toprule
  \multicolumn{1}{p{1.5cm}|}{\centering \textbf{Decentralized algorithm}} 
& \multicolumn{1}{p{3.5cm}|}{\centering \textbf{Per node, per iteration computational complexity}} 
& \multicolumn{1}{p{2cm}|} {\centering \textbf{Per node, per iteration communication complexity}}
& \multicolumn{1}{p{1.5cm}|} {\centering \textbf{Privacy of local signal estimates}}
& \multicolumn{1}{p{1.4cm}|} {\centering \textbf{Tunable parameters \\ (if any)}}
& \multicolumn{1}{p{2cm}} {\centering \textbf{Assumes \\ a priori knowledge of sparsity level}}
\\ \hline
\midrule
$\hspace{0.1cm}$ DCSP \cite{Varshney14DCSP}                & $ \mathcal{O}(m n + \zeta n + k \log{n} + m^2)$                 & $\mathcal{O}(\zeta n + k \log{n})$     & Yes & None & \hspace{-0.8cm} Yes\\
$\hspace{0.1cm}$ DCOMP \cite{Wimalajeewa13DCOMP}           & $ \mathcal{O}(n \zeta + L)$                                     & $ \mathcal{O}(\zeta n + L)$ & Yes & None & \hspace{-0.8cm} Yes \\
$\hspace{0.1cm}$ DRL-1 \cite{Ling_13_jsm_lqnorm}           & $ \mathcal{O}((n^2 + m^3 + nm^2)r_{\text{max}} + \zeta n)$   & $ \mathcal{O}(\zeta n)$     &  Yes &  Yes & \hspace{-0.8cm} No \\
$\hspace{0.1cm}$ DR-LASSO \cite{LingT11DecentralizedSuppportDetect} & $ \mathcal{O}(n^{2} m T_{1} + \zeta n T_{2})$ 
                                                                                                                             & $\mathcal{O}(\zeta n  T_{2})$     & Yes & Yes & \hspace{-0.8cm} No\\
$\hspace{0.1cm}$ Turbo-BCS \cite{Yang10TurboBCS}           & $\mathcal{O}(n^3 + nL + nk^{2} + k^{3} + mk)$                                                             & $ \mathcal{O}(kL)$     & No & None & \hspace{-0.8cm} No\\   
$\hspace{0.1cm}$ DCS-AMP \cite{Makhzani13DCSAMP}           & $ \mathcal{O}(mn + \zeta n + c_{1}n)$                           & $ \mathcal{O}(\zeta n)$    & Yes & Yes & \hspace{-0.8cm} No\\ 
$\hspace{0.1cm}$ CB-DSBL (proposed)                        & $ \mathcal{O}(n^2 + m^3 + nm^2 + \zeta n r_{\text{max}})$       & $ \mathcal{O}(\zeta n r_{\text{max}})$     & Yes & None & \hspace{-0.8cm} No \\
\hline 
\bottomrule
\multicolumn{6}{l}{\scriptsize{1.  $n, m, k$ and $L$  stand for the dimension of unknown sparse vector, number of local measurements per node, number of nonzero coefficients}}\\
\multicolumn{6}{l}{\scriptsize{  in the true support and network size, respectively.}}\\
\multicolumn{6}{l}{\scriptsize{2. $\zeta$ is the maximum number of communication links activated per node, per communication round.}}\\
\multicolumn{6}{l}{\scriptsize{3. $r_{\text{max}}$ is the number of inner loop ADMM iterations executed per CB-DSBL iteration.}}\\
\multicolumn{6}{l}{\scriptsize{4. $r_{\text{max}}$ is also the number of ADMM iterations used to obtain an inexact solution to the weighted $\ell_{1}$ norm based subproblem}} \\
\multicolumn{6}{l}{\scriptsize{   in the inner loop of DRL-1. }}\\
\multicolumn{6}{l}{\scriptsize{5. $T_{1}$ and $T_{2}$ denote the number of iterations of the two different inner loop iterations executed per DR-LASSO iteration.}}
\end{tabular*}
\end{center}
\end{table*}

\subsection{Contributions}
Our main contributions in this work are as follows:
\begin{enumerate}
\item We propose a novel decentralized, iterative, Bayesian joint-sparse signal recovery algorithm called \emph{Consensus Based Distributed Sparse Bayesian Learning}
or CB-DSBL. CB-DSBL works by establishing network wide consensus with respect to the estimated parameters of a joint sparsity inducing signal prior. 
The learnt signal prior is subsequently used by the individual nodes to obtain MAP estimates of local sparse signal vectors by the individual nodes. 
The proposed CB-DSBL algorithm does not require direct exchange of either local measurements or signal estimates between the nodes and hence is well
suited for applications where it is important to preserve the privacy of the local signal coefficients. 
\item The proposed algorithm employs the Alternating Directions Method of Multipliers (ADMM) to solve a series of iteration dependent consensus optimization problems
which require the nodes to exchange messages with each other. To reduce the associated communication overheads, we adopt a bandwidth efficient inter-node communication
scheme. This scheme entails the nodes exchanging messages with only a predesignated subset of its single hop neighbors known as \emph{bridge nodes}, as motivated
in \cite{Giannakis_08_NoisyLinks}. By selecting these \emph{bridge nodes}, one can trade off between communication bandwidth requirements and the ADMM's 
robustness to node failures. In this connection, we analytically establish the relationship between the selected set of bridge nodes and the convergence rate of the 
ADMM iterations. For the bridge-node based inter-node communication scheme, we show linear rate of convergence for the ADMM iterations when applied to a generic consensus optimization problem. 
The analysis is useful in obtaining a closed form expression for the tunable parameter of our proposed joint sparse signal recovery algorithm, ensuring its fast convergence.
\item We empirically demonstrate the superior MSE and support recovery performance of CB-DSBL in comparison to existing decentralized algorithms: DRL-1, DCOMP 
and DCSP.  
\end{enumerate}
In Table \ref{tab:table_jss_algo_compare}, we compare the existing decentralized joint-sparse signal recovery schemes with respect to their per iteration computational and communication complexity, privacy of local estimates,
presence/absence of tunable parameters and dependence on prior knowledge of the sparsity level. 
As highlighted in the comparison in Table \ref{tab:table_jss_algo_compare}, CB-DSBL belongs to a handful of decentralized algorithms for joint-sparse signal recovery which do not require a priori knowledge of 
the sparsity level, rely only on single hop communication, and do not involve direct exchange of local signal estimates between network nodes. 
Besides this, unlike loopy Belief Propagation (BP) or Approximate Message Passing (AMP) based Bayesian algorithms, CB-DSBL does not suffer from any convergence
issues even when the local measurement matrix at each node is dense or not randomly constructed. 

The rest of this paper is organized as follows. Section~\ref{sec:system_model} describes the system model and the problem statement of distributed JSM-2 signal recovery.
Section \ref{sec:centralized_algo} discusses centralized M-SBL \cite{Wipf_07_msbl} adapted to our setup, and sets the stage for our proposed decentralized solution.
Section \ref{sec:cb_dsbl} develops the proposed CB-DSBL algorithm along with a detailed discussion on the convergence properties of the underlying ADMM iterations.
Other implementation specific issues are also discussed. Section \ref{sec:sim_results} compares the performance of proposed algorithm with existing ones with respect
to various performance metrics. Finally, section \ref{sec:conclusions} concludes the paper. 

\emph{Notation}: Boldface lowercase and uppercase alphabets are used to denote vectors and matrices, respectively. Script styled alphabet (for example $\mathcal{A}$) 
is used to denote a set. $|\mathcal{A}|$ denotes the cardinality of set $\mathcal{A}$.
The term $\vecx_{j}^{k}(i)$ denotes the $i^{\text{th}}$ element of vector $\vecx$ associated with node $s_{j}$ at $k^{\text{th}}$ iteration/time index. The superscript $(.)^{T}$ denotes the transpose operation. For matrices $\matA$ and $\matB$ of sizes $m \times n$ and $p \times q$ respectively, $\matA \otimes \matB$ denotes their 
Kronecker product, which is of size $mp \times nq$.
$\mathcal{N}(\vmu, \matSigma)$ denotes the Gaussian distribution with mean $\vmu$ and covariance matrix $\matSigma$. 
$\mathbb{E(\vecx | \vecy)}$ denotes taking expectation of random variable $\vecx$ conditioned on another random variable $\vecy$.
 
\section{Distributed JSM-2 System Model}\label{sec:system_model}
We consider a network of $L$ nodes/sensors connected as a network described by a bi-directional graph 
$\mathcal{G} = (\mathcal{J}, \mathcal{A})$. $\J = \{1,2,\dots, L\}$ is the set of vertices in $\mathcal{G}$, each vertex representing a node in the network.
Set $\mathcal{A}$ contains the edges in $\mathcal{G}$, each edge representing a single hop error-free communication link between a distinct pair of nodes. 
Each node is interested in estimating an unknown $k$-sparse vector $\vecx_{j} \in \Real^{n}$ from $m$ locally acquired noisy linear measurements 
$\vecy_{j} \in \Real^{m}$. The generative model of the local measurement vector $\vecy_{j}$ at node $j$ is given by 
\begin{equation} \label{jsm2_meas_model}
 \vecy_{j} = \matPhi_{j}\vecx_{j} + \vecw_{j},  \;\; \;\; 1 \le j \le L
\end{equation}
where, $\matPhi_{j} \in \Real^{m \times n}$ is a full rank sensing matrix and $\vecw_{j} \in \Real^{m}$ is the measurement noise modeled as zero mean 
Gaussian distributed with covariance matrix $\sigma_{j}^{2}\mathbf{I}_{m}$. The sparse vectors $\vecxall$ at different nodes follow the JSM-2 signal model \cite{Duarte09DCS_main}. 
This implies that all $\vecx_{j}$ share a common support, represented by the index set $\mathcal{S}$. From the JSM-2 model, it also follows that the nonzero coefficients
of the sparse vectors are independent within and across the vectors.

The goal is to recover the local sparse vectors $\vecxall$ at their respective nodes using decentralized processing. 
In addition to processing the local data $\lc \vecy_{j}, \matPhi_{j}, \sigma_{j}^{2} \rc$, each node must collaborate with its single hop neighboring nodes to exploit the network wide joint sparsity of the unknown sparse vectors. For sake of privacy, the 
nodes are prohibited from directly exchanging their local measurements or local signal estimates.
Finally, the decentralized algorithm should be able to generate the centralized solution at each node, as if each node has access to the entire global information i.e., $\lc\vecy_{j}, \matPhi_{j}, \sigma_{j}^{2}\rc_{j \in \J}$. 

\section{Centralized Algorithm for JSM-2}\label{sec:centralized_algo}
In this section, we briefly recall the centralized M-SBL algorithm 
\cite{Wipf_07_msbl} for JSM-2 signal recovery and extend it to support distinct measurement matrices $\matPhi_{j}$ and noise variances $\sigma_{j}^{2}$ at each node. 
The centralized algorithm runs at an FC, which assumes complete knowledge of network wide information, $\lc \vecy_{j}, \matPhi_{j}, \sigma_{j}^{2} \rc_{j =1}^{L}$. 
For ease of notation, we introduce two variables 
$\matX \triangleq \lc \vecx_{1}, \vecx_{2}, \dots, \vecx_{L} \rc$ and $\matY \triangleq \lc \vecy_{1}, \vecy_{2}, \dots, \vecy_{L} \rc$ to be used in the sequel.

Similar to M-SBL, each of the sparse vectors $\vecx_{j}, j \in \J$ is assumed to be distributed according to a parameterized signal prior $p(\vecx_{j}; \vgamma)$ shown below.
\begin{eqnarray} \label{signal_prior_for_x}
 p(\vecx_{j}; \vgamma) &=& \displaystyle \prod_{i=1}^{n}p \lb \vecx_{j}(i);\vgamma(i) \rb 
 \nonumber  \\
 &=& \displaystyle \prod_{i=1}^{n} \frac{1}{\sqrt{2\pi\vgamma(i)}} \exp{ \lb -\frac{\vecx_{j}(i)^{2}}{2\vgamma(i)} \rb}.
\end{eqnarray}
Further, the joint signal prior $p(\matX;\vgamma)$ is assumed to be given by 
\begin{equation} \label{joint_signal_prior_for_X}
 p(\matX; \vgamma) = \prod_{j \in \J} p(\vecx_{j}; \vgamma).
\end{equation}
In the above, $\vgamma = \lb \vgamma(0), \vgamma(1), \dots, \vgamma(n) \rb^{T}$
is an $n$ dimensional hyperparameter vector, whose $i^{\text{th}}$ entry, $\vgamma(i)$, models the common variance of $\vecx_{j}(i)$ for $1 \le j \le L$.  
Since the signal priors $p(\vecx_{j}; \vgamma)$ are parameterized by a common $\vgamma$, if $\vgamma$ has a sparse support $\mathcal{S}$, 
then the MAP estimates of $\vecxall$ will also be jointly sparse with the same common support $\mathcal{S}$. The Gaussian prior
in (\ref{signal_prior_for_x}) promotes sparsity as it has an alternate interpretation as a parameterized model for the family of variational approximations
to a sparsity inducing Student's t-distributed prior \cite{Wipf03perspectives}. Under this interpretation, finding the hyperparameter vector $\vgamma$ which
maximizes the likelihood $p(\matY;\vgamma)$ is equivalent to finding the variational approximation which has the largest Bayesian evidence.  

Let $\hat{\vgamma}_{\text{ML}}$ denote the maximum likelihood (ML) estimate of hyperparameters of the joint source prior:
\begin{equation} \label{vgamma_ML}
 \hat{\vgamma}_{\text{ML}} = \underset{\vgamma}{\text{arg max }} p(\matY; \vgamma)
\end{equation}
where $p(\matY; \vgamma)$ is a type-2 likelihood function obtained by marginalizing the joint density $p(\matY, \matX ; \vgamma)$ 
with respect to the unknown vectors in $\matX$, i.e.,
\begin{eqnarray}
p(\matY ; \vgamma) &= \displaystyle \prod_{j =1}^{L} \int p(\vecy_{j} | \vecx_{j}) p(\vecx_{j} ; \vgamma) d\vecx_{j}  \nonumber \\
&= \displaystyle \prod_{j = 1}^{L} \mathcal{N} \lb 0, \matPhi_{j} \matGamma \matPhi_{j}^{T} + \sigma_{j}^{2}\matI_{m} \rb.   
\label{marginalized_likelihood}
\end{eqnarray}
Here $\matGamma = \text{diag}(\vgamma)$. We note that  $\hat{\vgamma}_{\text{ML}}$ cannot be derived in closed form by directly maximizing the likelihood in (\ref{marginalized_likelihood}) with respect to $\vgamma$. Hence, as suggested in the SBL framework \cite{Wipf_04_sbl},
we use the expectation maximization (EM) procedure to maximize $\log{p(\matY; \vgamma)}$ by treating $\matX$ as hidden variables. 

We now discuss the main steps of the EM algorithm to obtain $\hat{\vgamma}_{\text{ML}}$. Let $q_{\theta}(\matX)$
denote the variational approximation of true conditional density $p(\matX | \matY, \vgamma)$ with variational parameter set $\theta = (\tilde{\vmu}_{j}, \tilde{\matSigma}_{j})_{j \in \J}$. 
The variational parameters $\tilde{\vmu}_{j}$ and $\tilde{\matSigma}_{j}$ represent the conditional mean and covariance of $\vecx_{j}$ given $\vecy_{j}$.  
Then, as shown in \cite{Neal98IncrementalEM}, the log likelihood admits the following decomposition. 
\begin{eqnarray} \label{log_likelihood_decompose}
\displaystyle \log{p(\matY ; \vgamma)} 
&=& \int q_{\theta}(\matX) \log{\frac{p(\matY, \matX ; \vgamma)}{q_{\theta}(\matX)}} d\matX  
\; 
\nonumber \\
&+& 
\; D \lb q_{\theta}(\matX) \; || \; p(\matX | \matY;\vgamma) \rb
\end{eqnarray}
where the term $D(q_{\theta}||p) = \int q_{\theta}(\matX) \log{\frac{q_{\theta}(\matX)}{p(\matX | \matY ; \vgamma)}} d\matX$ is the 
\emph{Kullback-Leibler} (KL) divergence between the probability densities $q_{\theta}$ and $p$. From the non-negativity of $D(q_{\theta}||p)$ \cite{CoverThomasBook},
the log likelihood is lower bounded by the first term in the RHS. In the \emph{E-step}, we choose $\theta$ to make this variational lower bound tight by minimizing
the KL divergence term.
\begin{equation} \label{estep_theory}
\theta^{k+1} = \underset{\theta}{\text{arg} \; \text{min}} \;\; D(q_{\theta}(\matX) \;||\; p(\matX | \matY,\vgamma^{k})). 
\end{equation}
Here, $k$ denotes the iteration index of EM algorithm. From LMMSE theory, $p(\vecx_{j} | \vecy_{j}, \vgamma^{k})$ is Gaussian 
with mean $\vmu_{j}^{k+1}$ and covariance $\matSigma_{j}^{k+1}$ given by
\begin{eqnarray} 
 && \matSigma_{j}^{k+1} = \mathbf{\Gamma}^{k} - \mathbf{\Gamma}^{k}\matPhi_{j}^{T} \lb \sigma_{j}^{2}\matI_{m} + \matPhi_{j} \mathbf{\Gamma}^{k}\matPhi_{j}^{T} \rb ^{-1} \matPhi_{j} \mathbf{\Gamma}^{k} 
 \nonumber \\
 && \text{and } \vmu_{j}^{k+1} = \sigma_{j}^{-2}\matSigma_{j}^{k+1}\matPhi_{j}^{T}\vecy_{j}. \label{lmmse}
\end{eqnarray}
By choosing $\theta^{k+1} = \{ \vmu_{j}^{k+1}, \matSigma_{j}^{k+1} \}_{j \in \J}$ and $\displaystyle q_{\theta^{k+1}}(\matX) \sim \prod_{j \in \J}
\mathcal{N}(\vecx_{j}; \vmu_{j}^{k+1}, \matSigma_{j}^{k+1})$, the KL divergence term in (\ref{estep_theory}) can be driven to its minimum value of zero. 

In the \emph{M-step}, we choose $\vgamma$ to maximize the tight variational lower bound obtained in the \emph{E-step}:
\begin{eqnarray} \label{mstep_theory}
\vgamma^{k+1} =&& \hspace{-0.5cm}
\displaystyle \underset{\vgamma}{\text{arg max}}
\int q_{\theta^{k+1}}(\matX) \log{\frac{p(\matY, \matX ; \vgamma)}{q_{\theta^{k+1}}(\matX)}} \mathrm{d}\matX 
\nonumber \\
=&& \hspace{-0.5cm}
\underset{\vgamma}{\text{arg max}} \;
\displaystyle \mathbb{E}_{\matX \sim q_{\theta^{k+1}}} \left [\log{p(\matY, \matX ; \vgamma)} \right]. 
\end{eqnarray}
As shown in Appendix \ref{App:appendix_mstep_cf}, the optimization problem (\ref{mstep_theory}) can be recast as the following minimization problem.
\begin{equation} \label{m_step_cbdsbl}
 \vgamma^{k+1} = 
 \underset{\vgamma \in \Real_{+}^{n}}{\text{arg min}} \displaystyle \sum_{j \in \J} \sum_{i = 1}^{n} 
 \lb
 \log{\vgamma(i)}  
+
\frac{ \matSigma_{j}^{k}(i,i) + \vmu_{j}^{k}(i)^{2}}{\vgamma(i)} 
\rb.
\end{equation}
From the zero gradient optimality condition in (\ref{m_step_cbdsbl}), the M-step reduces to the following update rule: 
  \begin{equation} \label{m_step_update}
  \vgamma^{k+1}(i) = \frac{1}{L} \sum_{j \in \J} \lb \matSigma_{j}^{k+1}(i,i) + \vmu_{j}^{k+1}(i)^{2} \rb  \;\;\;\; \text{for } 1 \leq i \leq n.	 
\end{equation}
By repeatedly iterating between the E-step (\ref{lmmse}) and the M-step (\ref{m_step_update}), the EM algorithm converges to either a local maxima or a
saddle point of $\log{p(\matY|\vgamma)}$ \cite{Dempster77EM}. Once $\hat{\vgamma}_{\text{ML}}$ is obtained, the MAP estimate of $\vecx_{j}$ is evaluated by substituting it in the
expression for $\vmu_{j}$ in (\ref{lmmse}). It is observed that when the EM algorithm converges, the $\vgamma(i)$'s belonging to the inactive support tend to zero,
resulting in sparse MAP estimates. In practice, hard thresholding of $\vgamma$ is required to identify the nonzero support set. 
In this work, we remove all coefficients from the active support set for which $\vgamma(i), 1 \leq i \leq n$ is below the local noise variance.
It must be noted that if the local noise variance at each node is unknown, it can be estimated along with $\vgamma$ within the EM framework, as discussed in \cite{Wipf_07_msbl}. 

\section{Decentralized Algorithm for JSM-2}\label{sec:cb_dsbl}
\subsection{Algorithm Development}\label{sec:decentralized_algo_dev}
In this section, we develop a decentralized version of the centralized algorithm discussed in the previous section. For notational convenience, 
we introduce an $n$ length vector $\linebreak \veca_{j}^{k} = \lb a_{j,1}^{k}, a_{j,2}^{k}, \dots, a_{j,n}^{k} \rb^{T} $ maintained at node $j$, where 
$a_{j,i}^{k} = \matSigma_{j}^{k}(i,i) + \vmu_{j}^{k}(i)^{2}$, $\matSigma_j^k$ and $\vmu_j^k$ are as defined in~\eqref{lmmse}. 

From (\ref{m_step_update}), we observe that the solution of the M-step optimization (\ref{m_step_cbdsbl}) can be interpreted as an average of the $L$ vectors 
$\lc \veca_{j}^{k+1} \rc_{j = 1}^{L}$. The same solution can also be obtained by solving a different minimization problem
\begin{equation} \label{equivalent_avg_prob}
 \vgamma^{k+1} = \underset{\vgamma \in \Real_{+}^{n}}{\text{arg min }} \sum_{j \in \J} \norm{\vgamma - \veca_{j}^{k+1} }_{2}^{2}.
\end{equation}
Unlike the non-convex M-step objective function in (\ref{m_step_cbdsbl}), the surrogate objective function in (\ref{equivalent_avg_prob}) is convex in 
$\vgamma$ and therefore can be minimized in a distributed manner using powerful convex optimization techniques. An alternate form of (\ref{equivalent_avg_prob}) 
amenable to distributed optimization is given by
\begin{eqnarray}
&& \underset{\vgamma_{j} \in \Real_{+}^{n}, \; j \in \J}{\text{min    }} \; \sum_{j \in \J} \norm{\vgamma_{j} - \veca_{j}^{k+1} }_{2}^{2} \nonumber \\
&& \text{subject to } \vgamma_{j} = \vgamma_{j^{'}} \;\;\;\;  \forall \;\; j \in \J, \; j^{'} \in \N_{j}
\label{m_step_consensus_opt}
\end{eqnarray}
where $\N_{j}$ denotes the set of single hop neighbors of node $j$. The equality constraints in (\ref{m_step_consensus_opt}) ensure its equivalence 
to the unconstrained optimization in (\ref{equivalent_avg_prob}). Here, the number of equality constraints is equal to $|\mathcal{A}|$, i.e., the total number of single hop
links in the network. In a conventional decentralized implementation of (\ref{m_step_consensus_opt}), the number of messages exchanged between the nodes grow linearly with the 
number of consensus constraints. By restricting the nodes to exchange information only through a relatively small set of pre-designated nodes called \emph{bridge nodes},
the number of consensus constraints can be drastically reduced without affecting the equivalence of (\ref{equivalent_avg_prob}) and (\ref{m_step_consensus_opt}). Let $\B \subseteq \J$ denote the set of 
all bridge nodes in the network and $\B_{j} \subseteq \B$ denote the set of bridge nodes belonging to the single hop neighborhood of node $j$, then (\ref{m_step_consensus_opt})
can be rewritten as
\begin{eqnarray}
&& \underset{\vgamma_{j} \in \Real_{+}^{n}, j \in \J}{\text{minimize }} \; \sum_{j \in \J} \norm{\vgamma_{j} - \veca_{j}^{k+1} }_{2}^{2} \nonumber \\
&& \text{subject to } \vgamma_{j} = \vgamma_{b} \;\;\;\;  \forall \;\; j \in \J, \; b \in \B_{j}.
\label{m_step_cbdsbl_mod}
\end{eqnarray}
The auxiliary variables $\vgamma_{b}$, called \emph{bridge parameters}, are used to establish consensus among $\vgamma_{j}$.
Each bridge parameter $\vgamma_{b}$ is a non negative $n$ length vector maintained by the bridge node $b$. As motivated in \cite{Giannakis_08_NoisyLinks}, \cite{Giannakis_08_CBDEM}, 
using bridge nodes to impose network wide consensus allows us to trade off between the communication cost and robustness of the distributed optimization algorithm.\footnote
{In an alternate embodiment of the proposed algorithm, the message exchanges could be restricted to occur only through the (trustworthy) bridge nodes, thereby avoiding direct communication between the nodes. In this case, the role of the bridge nodes could be to enforce consensus in $\vgamma$ across the nodes, and these nodes need not directly participate in signal reconstruction.} 

The following Lemma provides sufficient conditions on the choice of the bridge node set $\B$ under which 
(\ref{equivalent_avg_prob}) and (\ref{m_step_cbdsbl_mod}) are equivalent.  
The proof for the Lemma can be found in \cite{Giannakis_08_NoisyLinks}.
\begin{lemma} \label{bridge_node_sel_rule}
For a connected graph $\mathcal{G}$, if the bridge node set $\B \subseteq \J$ satisfies the following conditions
\begin{enumerate}
 \item Each node $s_{j}$ must be connected to at least one bridge node in $\B$, i.e., $\B_{j} \neq \phi$ for any $j \in \J$, and,
 \item If two nodes $s_{j_{1}}$ and $s_{j_{2}}$ are single-hop neighbors, then $\B_{j_{1}} \bigcap \B_{j_{2}} \neq \phi$ for any $j_{1}, j_{2} \in \J$, 
\end{enumerate}
then, in the solution to (\ref{m_step_cbdsbl_mod}), $\vgamma_{j}$'s are equal for all $j \in \J$ .
\end{lemma}
\begin{figure}
\centering
\includegraphics[scale = 0.65]{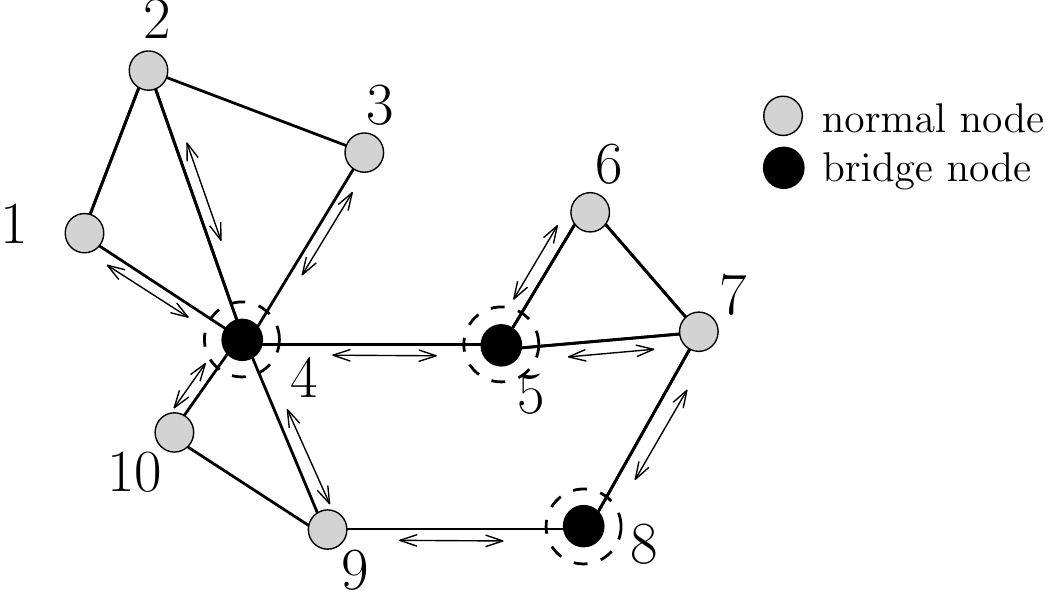}
\caption{Selection of bridge nodes in a sample network consisting of 10 nodes. In the proposed scheme, only those edges that have at least one of the vertices as a bridge node are used for communication. The remaining edges are not used for communication. For example, node $9$ communicates only with bridge nodes $4$ and $8$.}
\label{fig_sample_network}
\end{figure}
Fig.~\ref{fig_sample_network} illustrates the selection of bridge nodes according to Lemma \ref{bridge_node_sel_rule}, in a sample network.  
In this work, we employ the \emph{Alternating Directions Method of Multipliers} (ADMM) algorithm \cite{Parikh_11_ADMM} to solve the convex optimization problem
in (\ref{m_step_cbdsbl_mod}). ADMM is the state of the art dual ascent algorithm for solving constrained convex optimization problems,
offering a linear convergence rate and a natural extension to a decentralized implementation. 

We start by constructing an augmented Lagrangian, $L_{\rho}$, given by 
\begin{eqnarray} \label{augmented_lagrangian}
 L_{\rho}(\vgamma_{\J}, \vgamma_{\B}, \vlambda) 
 &\triangleq& 
\sum_{j \in \J} \norm{\vgamma_{j} - \veca_{j}^{k+1}}_{2}^{2} +
 \nonumber \\ 
&& \hspace{-3cm}
\sum_{j \in \J} \sum_{b \in \B_{j}} (\vlambdajb)^{T}(\vgammaj - \vgammab) + 
\frac{\rho}{2} \sum_{j \in \J} \sum_{b \in \B_{j}} \norm{\vgammaj - \vgammab}_{2}^{2} 
\end{eqnarray}
where $\vlambdajb$ denotes the $n \times 1$ sized Lagrange multiplier vector corresponding to the equality constraint $\vgammaj = \vgammab$ and 
$\rho$ is a positive scalar which biases the quadratic consensus penalty term. For ease of notation, we define concatenated vectors 
$\vgamma_{\J} = \{\vgamma_{1}^{T}, \vgamma_{2}^{T}, \dots, \vgamma_{L}^{T}\}^{T}$ and $\vgamma_{\B} = \{ \vgamma_{b_{1}}^{T}, \dots, \vgamma_{b_{|\B|}}^{T}\}^{T}$
to be used in the sequel.
We also define the $nN_{C} \times 1$ concatenated Lagrange multiplier vector $\vlambda$, where $N_{C}$ is the number of equality constraints in (\ref{m_step_cbdsbl_mod}).
The solution to (\ref{m_step_cbdsbl_mod}) is then obtained by executing the following ADMM iterations until convergence:
\begin{eqnarray}
\label{admm_iteration1}
&& \vgamma_{\J}^{r+1} = \underset{\vgamma_{\J}}{\text{arg min}} \;  L_{\rho}(\vgamma_{\J}, \vgamma_{\B}^{r}, \vlambda^{r})   \\ 
\label{admm_iteration2}
&& \vgamma_{\B}^{r+1} = \underset{\vgamma_{\B}}{\text{arg min}} \;  L_{\rho}(\vgamma_{\J}^{r+1}, \vgamma_{\B}, \vlambda^{r})   \\
\label{admm_iteration3}
&& (\vlambda_{j}^{b})^{r+1} = (\vlambda_{j}^{b})^{r} + \rho(\vgammaj^{r+1} - \vgammab^{r+1}) 
\end{eqnarray}
$\forall j \in \J, b \in \B_{j}$. Here, $r$ denotes the ADMM iteration index. In (\ref{admm_iteration1}-\ref{admm_iteration2}), the primal variables,
$\vgamma_{\J}$ and $\vgamma_{\B}$, are updated in a Gauss-Seidel fashion by minimizing the augmented Lagrangian, $L_{\rho}$, evaluated at the previous
estimate of the dual variable $\vlambda$. By adding an extra quadratic penalty term to the original Lagrangian, the objective in (\ref{admm_iteration2})
is no longer affine in $\vgamma_{\B}$ and hence has a bounded minimizer. The dual variable $\vlambda$ is updated via a gradient-ascent step
(\ref{admm_iteration3}) with a step-size equal to the ADMM parameter $\rho$. This particular choice of step-size ensures the dual feasibility of the
iterates $\{\vgamma_{\J}^{r+1}, \vgamma_{\B}^{r+1}, \vlambda^{r+1} \}$ for all $r$. Since the augmented Lagrangian $L_{\rho}$ is strictly convex with
respect to $\vgamma_{\J}$ and $\vgamma_{\B}$ individually, the zero gradient optimality conditions for (\ref{admm_iteration1}) and (\ref{admm_iteration2})
translate into simple update equations for $\vgamma_{j}$ and $\vgamma_{b}$:
\begin{equation}
\label{vgammaj_update}
\vgamma_{j}^{r+1} =  \frac{2 \veca_{j}^{k+1} + \sum_{b \in \B_{j}} \lb \rho \vgammab^{r} - (\vlambdajb)^{r} \rb}{2 + \rho |\B_{j}|} \;\;\;\; \forall \; j \in \J
\end{equation}
\begin{equation} \label{vgammab_update}
\text{and } \;\;\vgamma_{b}^{r+1} = \frac{\sum_{j \in \N_{b}} (\rho \vgamma_{j}^{r+1} + (\vlambdajb)^{r})}{\rho|\N_{b}|} \;\;\;\;\forall \; b \in \B.
\end{equation}
Here $\N_{b}$ denotes the set of nodes connected to bridge node $b$. As shown in Appendix \ref{App:simplify_admm_iters}, by eliminating the Lagrange multiplier terms from \eqref{admm_iteration3} and \eqref{vgammab_update}, the update rule for $\vgamma_{b}$ can be 
further simplified to 
\begin{eqnarray}
\label{vgammab_update_simple}
&& \vgamma_{b}^{r+1} = \frac{1}{|\N_{b}|} \sum_{j \in \N_{b}} \vgamma_{j}^{r+1} \hspace{0.5 cm} \forall \; b \in \B.
\end{eqnarray}

In section \ref{sec:cbdsbl_variants}, we compare the bridge node based ADMM discussed above with other decentralized optimization techniques available
in the literature. We show empirically that the bridge node based ADMM scheme is able to flexibly trade off between communication complexity, robustness
to node failures, speed of convergence, and signal reconstruction performance. 

\subsection{CB-DSBL Algorithm}\label{sec:cbdsbl_algorithm}
We now propose the CB-DSBL algorithm. Essentially, it is a decentralized EM algorithm for finding the ML estimate of the hyperparameters $\vgamma$.
The algorithm comprises two nested loops. In the outer loop, each node performs the E-step (\ref{lmmse}) in a standalone manner. In the inner loop,
ADMM iterations are performed to solve the M-step optimization in a decentralized manner. Upon convergence of the outer loop, each node $j \in \J$
has the same ML estimate of $\vgamma$, which is then used to obtain a MAP estimate of the local sparse vector $\vecx_{j}$, similar to the centralized
algorithm. The steps of the CB-DSBL algorithm are detailed in Algorithm 1. 
\begin{algorithm}[ht] \label{algo:algo_summary}
\begin{tabular}{p{7.2cm}}
\caption{Consensus Based Distributed Sparse Bayesian Learning (CB-DSBL)}
\end{tabular}
\dontprintsemicolon
\vspace{-0.6cm}
\begin{small}
  {\textbf{Initializations: }$k \gets 0$} \; 
  {$\vgammaj^{k} \gets 10^{-3} \mathbf{1}_{n \times 1} \quad \forall j \in \J$} \;
  {$\vgammab^{k}$, $(\vlambdajb)^{k} \gets  0 \quad \forall j \in \J, \; b \in \B_{j}$} \;
  \vspace{0.4cm}
  \While{$\lb k < k_{\text{max}}\rb \& \lb \Delta \vgamma_{\J} > \epsilon \rb$}
 {
  \emph{E step:} Each node $s_{j}$, $j \in \J$,  updates $\veca_{j}^{k}$ according to (\ref{lmmse}).  
  \\ \vspace{0.1cm}
  \emph{M step:} $r \gets 0$, 
  $\vgamma_{\J}^{r} \gets \vgamma_{\J}^{k}$, 
  $\vgamma_{\B}^{r} \gets \vgamma_{\B}^{k}$,
  $(\vlambda)^{r} \gets (\vlambda)^{k}$ \\
  \While{$r < r_{\text{max}}$}
  {
    \hspace{-0.5cm}
    \begin{tabular}{p{0.25cm} p{6.4cm}}
    1. & All nodes $s_{j \in \J}$ update their local estimate of hyperparameters $\vgamma_{j}^{r}$ according to (\ref{vgammaj_update}).\\
    2. & All nodes $s_{j \in \J}$ transmit the updated $\vgammaj^{r+1}$ estimate to connected bridge nodes $s_{b \in \B_{j}}$.\\
    3. & Each bridge node $s_{b \in \B}$ updates its bridge variable $\vgamma_{b}^{r}$ according to (\ref{vgammab_update_simple}). \\
    4. & All bridge nodes $s_{b\in \B}$ transmit updated bridge hyperparameters $\vgamma_{b}^{r+1}$ to nodes in their neighborhood $\N_{b}$. \\
    5. & All nodes $s_{j \in \J}$ update their Lagrange multipliers $(\vlambdajb)^{r}, b \in \B_{j}$ according to (\ref{admm_iteration3}).\\
    6. & $r \gets r+1$
    \end{tabular}
  } 
  $\vgamma_{\J}^{k} \gets \vgamma_{\J}^{r}$,
  $\vgamma_{\B}^{k} \gets \vgamma_{\B}^{r}$, 
  $(\vlambda)^{k} \gets (\vlambda)^{r}$ \\
  $k \gets k+1$ \\
  $\Delta \vgamma_{\J} \gets ||\vgamma_{\J}^{k} - \vgamma_{\J}^{k-1}||_{2}$
 }
\end{small}
\end{algorithm}

Each ADMM iteration in the M-step of the CB-DSBL algorithm involves two rounds of communication (Steps $2$ and $4$) between the nodes. 
In the first communication round, each node $j \in J$ transmits $\vgamma_{j} \in \Real^{n}$
to its $|\B_{j}|$ single hop neighbors. In the second communication round, each bridge node $b \in \B$ transmits $\vgamma_{b} \in \Real^{n}$
to its $|\N_{b}|$ single hop neighbors. Thus, in each M-step, $2n\sum_{j \in \J}|\B_{j}|$ real numbers are exchanged between 
the nodes and their respective bridge nodes. 
In Fig.~\ref{fig:inner_loop_variations}, we compare different variants of CB-DSBL 
with respect to the average number of inter-node message exchanges required to achieve less than $1\%$ signal reconstruction error.
From the figure, it is evident that the aforementioned bridge node based ADMM technique is effective in reducing the overall inter-node communication and the associated costs, without compromising on signal reconstruction performance. 
One of the ways of selecting the bridge node set $\B$ is to sort the nodes in decreasing order of 
their nodal degrees and retain the least number of top most $|\B|$ nodes satisfying the conditions in Lemma \ref{bridge_node_sel_rule}. Although suboptimal, this scheme is able to significantly reduce the 
overall communication complexity of the algorithm as demonstrated empirically in Fig.~\ref{fig:inner_loop_variations}.
In section \ref{sec:rho_selection}, a rule of thumb policy is discussed to select the bridge nodes $\B$ which will ensure fast convergence 
of the decentralized ADMM iterations in the M-step of CB-DSBL algorithm.
\begin{figure}
\centering
\includegraphics[width=0.42\textwidth]{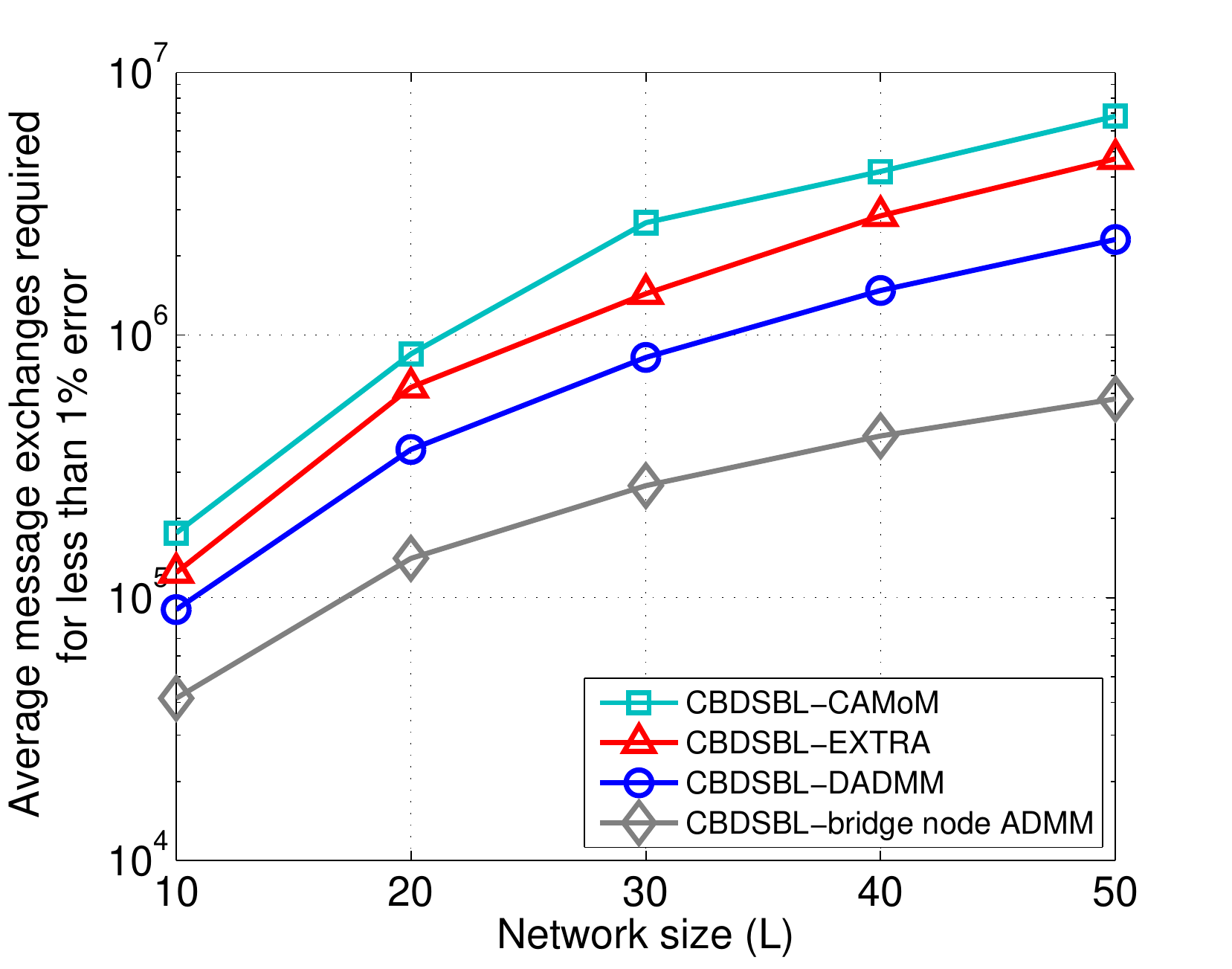}
\caption{Comparison of the communication complexity of CB-DSBL variants based on `bridge node' ADMM \cite{Giannakis_08_NoisyLinks}, CA-MoM \cite{HaoZhu09CAMoM}, D-ADMM \cite{Mota_13_ADMM} and EXTRA \cite{WeiShi15EXTRA} 
algorithms. The plot shows the average number of messages exchanged between nodes in order to achieve less than $1\%$ signal reconstruction 
error ($-20$ dB NMSE), The total number of message exchanges shown here is averaged across 500 trials. Other simulation parameters: $n = 50$,
$m = 10$, $10 \%$ sparsity, SNR = $30$ dB. 
}\label{fig:inner_loop_variations}
\end{figure}

Further reduction in inter-node communication is possible by executing only a finite number of ADMM iterations per M-step. In a practical embodiment of the algorithm, 
running a single ADMM iteration per M-step is sufficient for the CB-DSBL to converge. As shown in Fig.~\ref{fig:fig_subiter_regress}, beyond two or three 
ADMM iterations per M-step, there is only a marginal improvement in the quality of solution as well the convergence speed. Fig.~\ref{fig:fig_convergence_speed}
shows that even with a single ADMM iteration per M-step, CB-DSBL typically converges quite rapidly to the centralized solution.

\begin{figure}[h!t]
\centering
\includegraphics[width=0.38\textwidth]{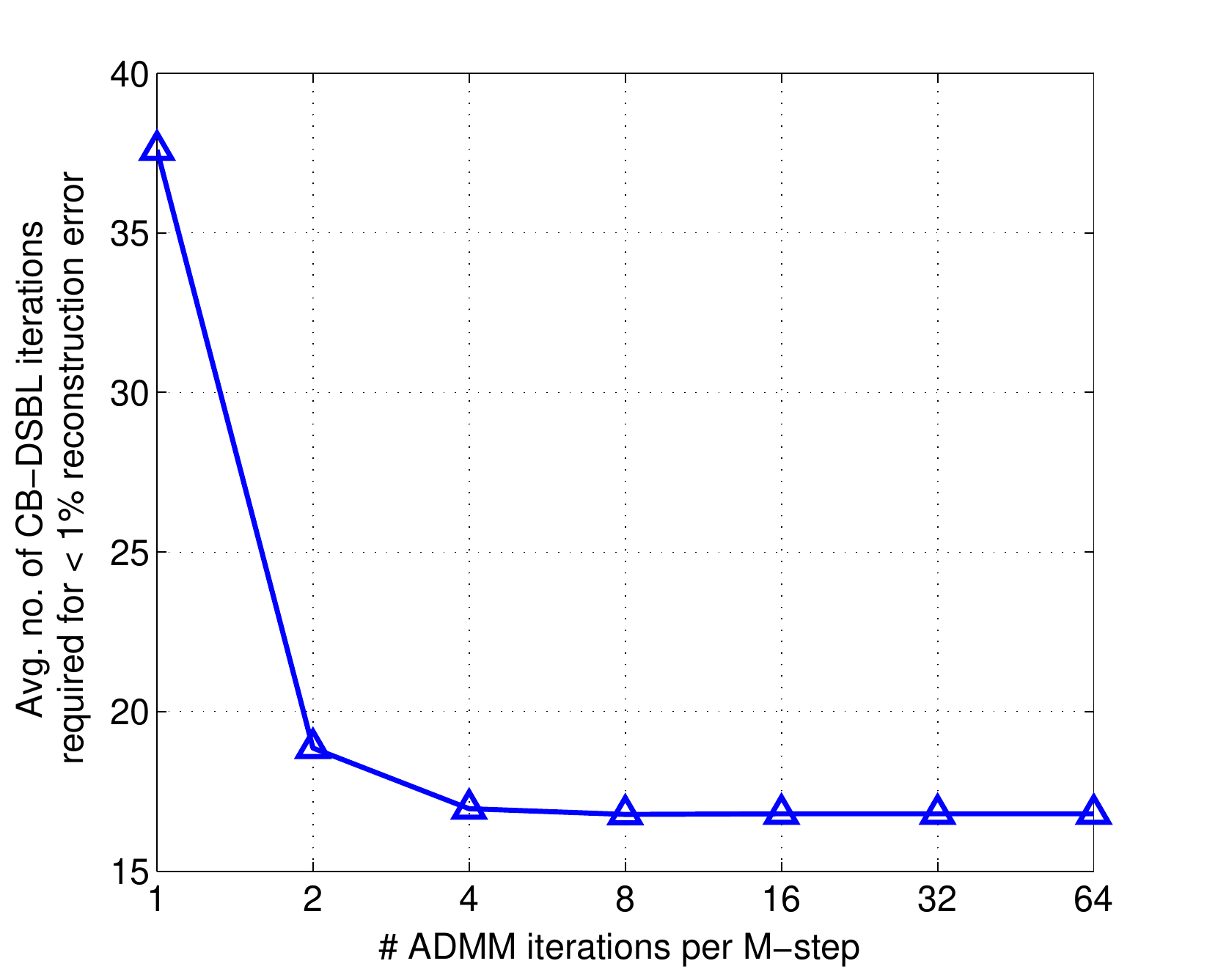}
\caption{This plot illustrates the sensitivity of CB-DSBL's outer loop iterations to the number of ADMM iterations executed per M-step in the inner loop of the algorithm. Each point in the curve represents the average number of overall CB-DSBL iterations needed to achieve less than 1$\%$ signal reconstruction error for a given number of ADMM iterations executed in the inner loop. Simulation parameters used: $n = 100$, $m = 10$, $L = 10$, $5 \%$ sparsity, SNR = $30$ dB and $\#$trials = $100$.
}\label{fig:fig_subiter_regress}
\end{figure}

\begin{figure}[h!t]
\begin{subfigure}[b]{0.24\textwidth}
    \includegraphics[width=\textwidth]{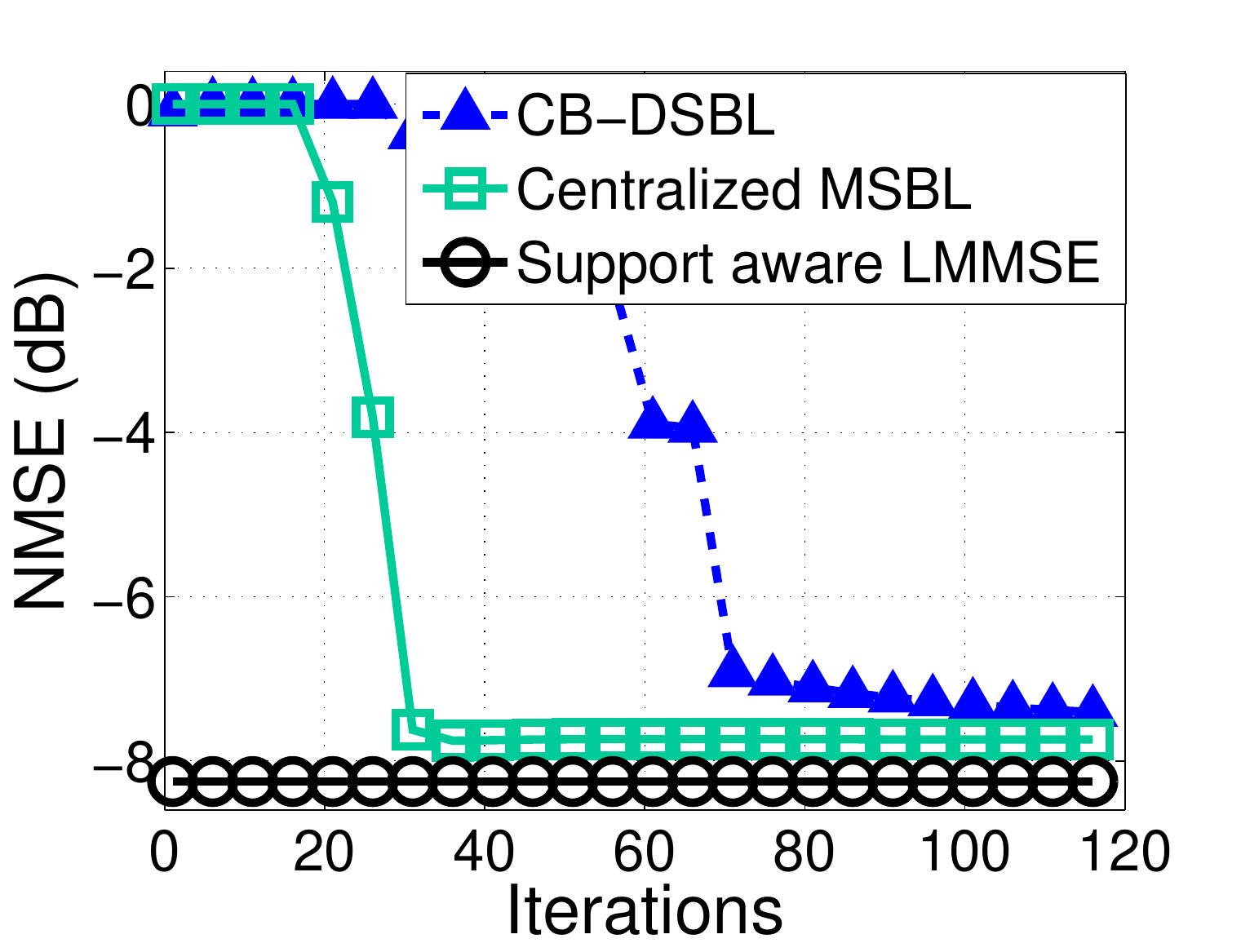}
    \caption{\scriptsize{$L = 10$ nodes, SNR $= 10$ dB}}
\end{subfigure}%
\begin{subfigure}[b]{0.24\textwidth}
    \includegraphics[width=\textwidth]{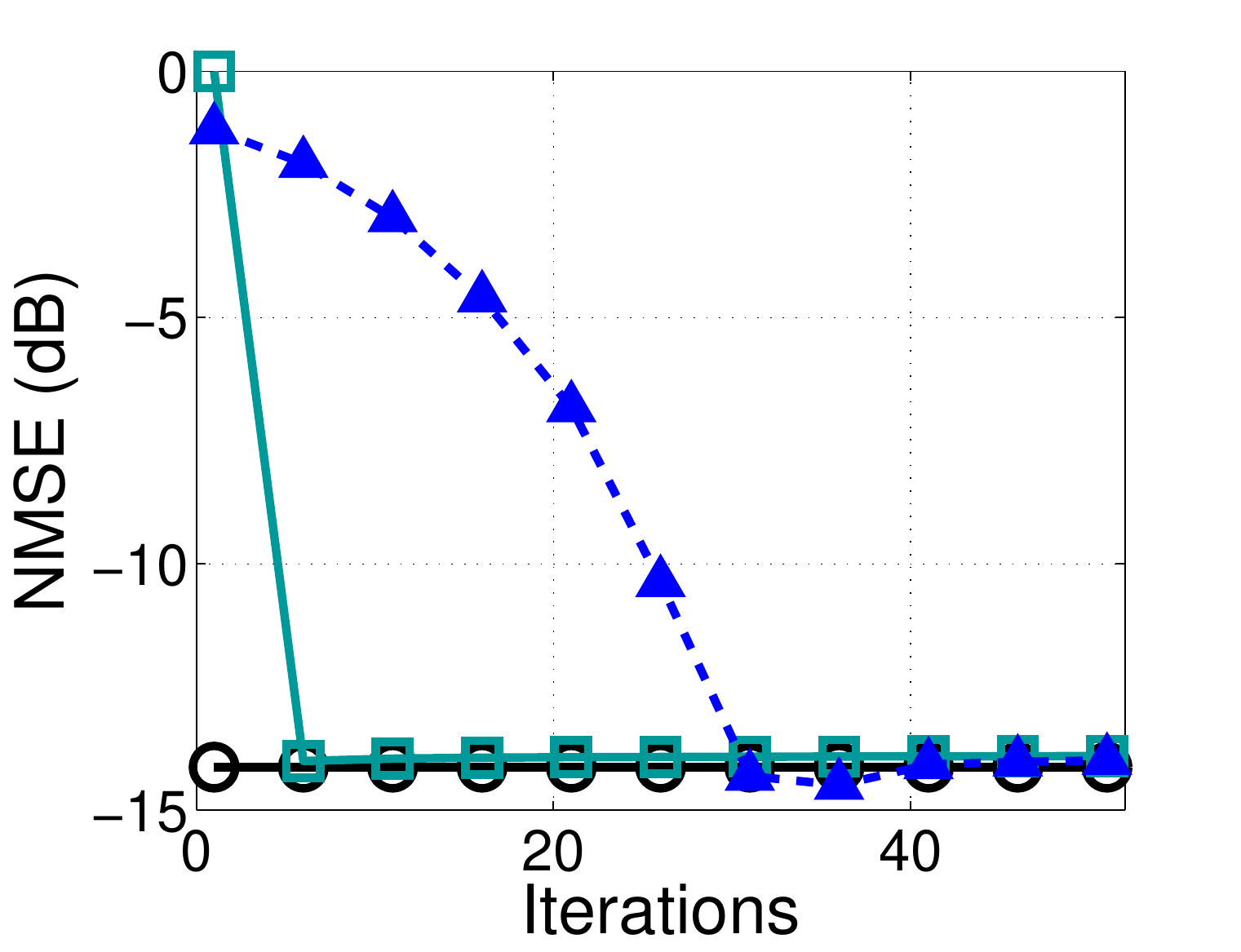}
    \caption{\scriptsize{$L = 10$ nodes, SNR $= 20$ dB}}
\end{subfigure}

\begin{subfigure}[b]{0.24\textwidth}
    \includegraphics[width=\textwidth]{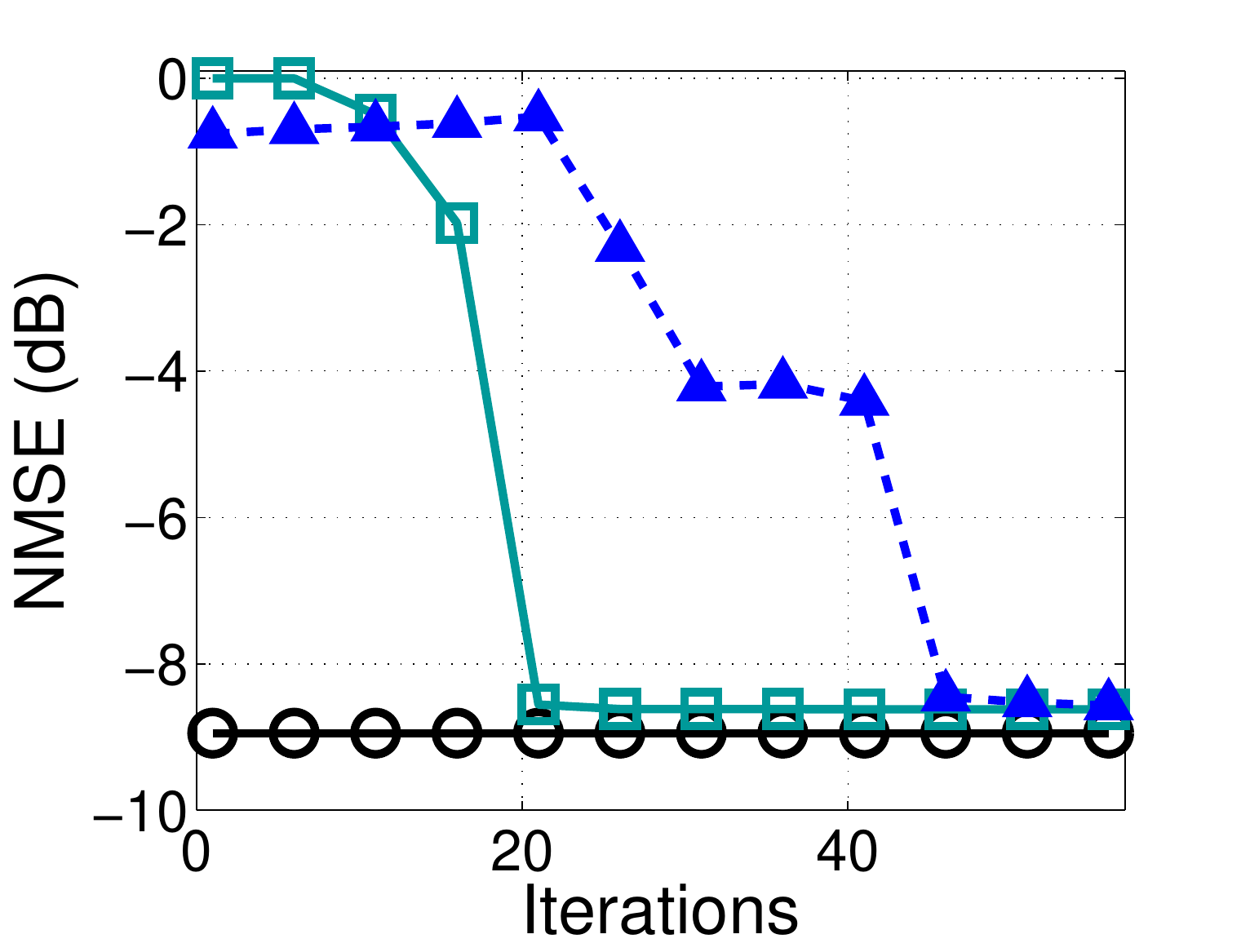}
    \caption{\scriptsize{$L = 20$ nodes, SNR $= 10$ dB}}
\end{subfigure}
\begin{subfigure}[b]{0.24\textwidth}
    \includegraphics[width=\textwidth]{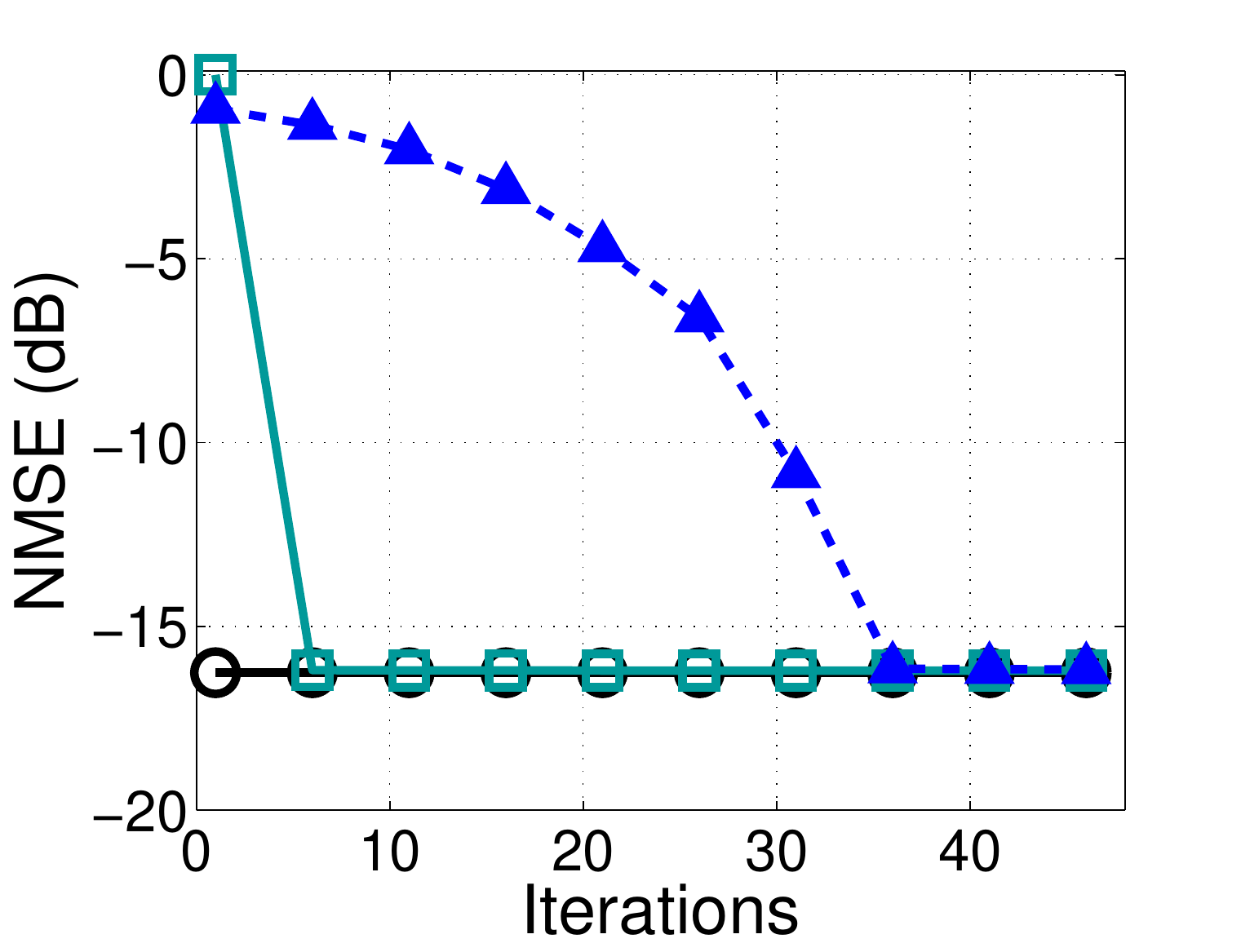}
    \caption{\scriptsize{$L = 20$ nodes, SNR $= 20$ dB}}
\end{subfigure}
\caption{Convergence of decentralized CB-DSBL to centralized M-SBL solution for different network sizes and SNRs. The CB-DSBL variant used here executes a single ADMM iteration per EM iteration. Other simulation parameters: $n = 50$, $m = 10$ and $10 \%$ sparsity.} 
\label{fig:fig_convergence_speed}   
\end{figure}

\subsection{Convergence of ADMM Iterations in the M-step}\label{sec:m_step_convergence}
In this section, we analyze the convergence of the ADMM iterations  (\ref{admm_iteration3}), (\ref{vgammaj_update}) and (\ref{vgammab_update_simple})
derived for the M-step optimization in CB-DSBL. By doing so, we aim to highlight the effects of the bridge node set $\B$ and the augmented Lagrangian parameter $\rho$
on the convergence of the ADMM iterations.

ADMM has been a very popular choice for solving both convex \cite{Matamoros151bitADMMJSM1, Ling_13_jsm_lqnorm, Parikh_11_ADMM, Mota_13_ADMM, Giannakis_08_NoisyLinks} and more recently nonconvex \cite{Erseghe15MLthruADMM} 
optimization problems as well, in a distributed setup. 
In its classical form, ADMM solves the following constrained optimization problem:
\begin{eqnarray}
& \underset{\vecx, \vecz}{\text{min }} f(\vecx) + g(\vecz) \nonumber \\
& \text{subject to } \matA \vecx + \matB \vecz = \vecc,
\label{classical_admm_prob}
\end{eqnarray}
where $\vecx \in \Real^{n}$ and $\vecz \in \Real^{m}$ are the primal variables. The matrices $\matA, \matB$ and the vector
$\vecc$ appearing in the linear equality constraint are of appropriate dimensions. The functions $f: \Real^{n} \to \Real$ and $g: \Real^{m} \to \Real$
are convex with respect to $\vecx$ and $\vecz$, respectively.
In \cite{WeiDengWotaoYinADMM12}, the authors have shown linear convergence rate for the classical ADMM iterations under the assumptions of strict convexity
and Lipschitz gradient on one of $f$ or $g$, along with full row rank assumptions for the matrix $\matA$. However, in the ADMM formulation of a decentralized
consensus optimization problem, the coefficient matrix $\matA$ is seldom of full row rank. In \cite{WataoYin13ADMMConvergence}, the full row rank condition of
$\matA$ was relaxed and linear rate of convergence was established for decentralized ADMM iterations for a generic convex optimization with linear consensus
constraints similar to \eqref{m_step_consensus_opt}. In \cite{Erseghe11AvgConsensusADMM}, the convergence of ADMM for solving an average consensus problem
has been analyzed for both noiseless and noisy communication links. In both \cite{WataoYin13ADMMConvergence} and \cite{Erseghe11AvgConsensusADMM},
the secondary primary variables indicated by the entries of $\vecz$ have a one to one correspondence with the communication links between the network nodes.
However, such a bijection is missing for the bridge variables used in our work for enforcing consensus between the primal variables. Due to this, the convergence
results of \cite{WataoYin13ADMMConvergence, Erseghe11AvgConsensusADMM} are not directly applicable to our case. In the sequel, we present the analysis
of the convergence of decentralized ADMM iterations for the bridge node internode communication scheme. 

In this section, we analyze the convergence of the ADMM iterations  (\ref{admm_iteration3}), (\ref{vgammaj_update}) and (\ref{vgammab_update_simple}) 
derived for the M-step optimization in CB-DSBL. By doing so, we aim to highlight the effects of the bridge node set $\B$ and the 
augmented Lagrangian parameter $\rho$ on the convergence of the ADMM iterations. 
We start by defining block matrices $\matE_{1} = \matC_{1} \otimes \matI_{n}$ and $\matE_{2} = \matC_{2} \otimes \matI_{n}$ 
of sizes $n N_{C} \times n L$ and $n N_{C} \times n |\B|$, respectively. 
The rows of $\matC_{1}$ and $\matC_{2}$ encode the $N_{C}$ equality constraints in (\ref{m_step_cbdsbl_mod}) such that 
if $i^{\text{th}}$ equality constraint is $\vgammaj = \vgamma_{b_{k}}$, $b_{k} \in \B$,  then $\matC_{1}(i,j) = 1$ and $\matC_{2}(i,k) = -1$;
with the rest of the entries in the $i^{\text{th}}$ row being zero. It can easily be shown that the minimum and maximum number of 
bridge nodes connected to any node in the network is the same as the minimum and maximum eigenvalues of $\matE_{1}^{T}\matE_{1}$, denoted by 
$\sigma^{2}_{\text{min}}$ and $\sigma^{2}_{\text{max}}$, respectively. Fig.~\ref{fig:fig_lin_constraint_matrix} illustrates the construction of the block matrices $\matE_{1}$ and $\matE_{2}$ for an example network consisting of $5$ nodes.
\begin{figure}
\centering
\includegraphics[scale = 0.66]{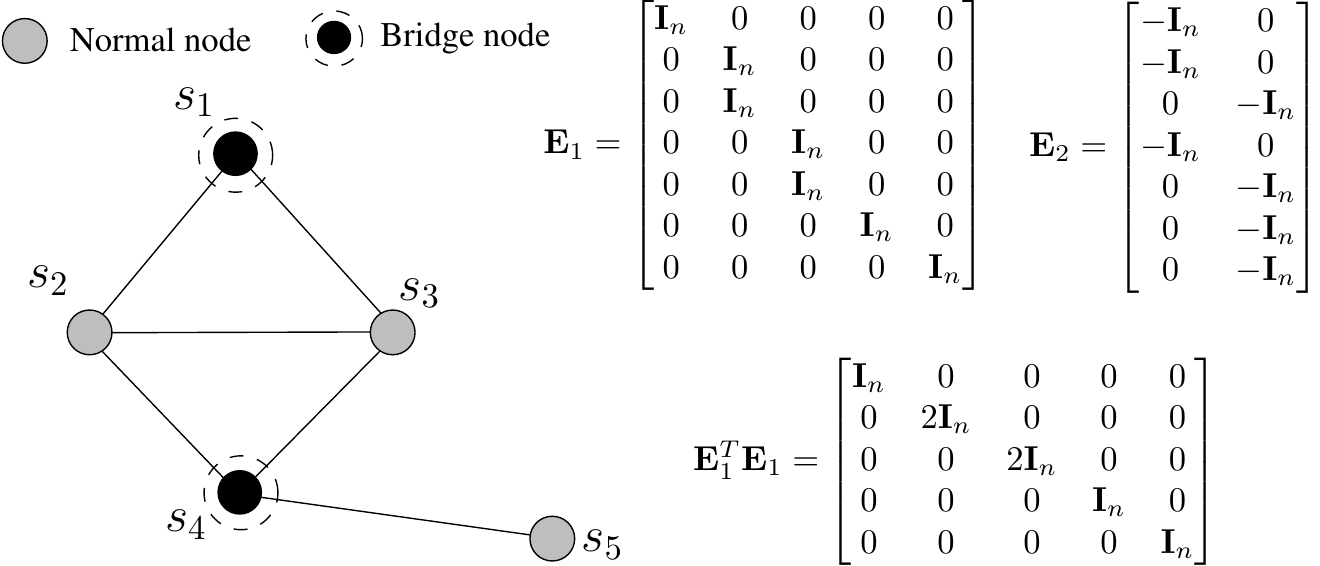}
\caption{Construction of block matrices $\matE_{1}$ and $\matE_{2}$ for a sample $5$ node network. The matrices $\matE_{1}$ and $\matE_{2}$ are together used to enforce the linear consensus constraints in \eqref{m_step_consensus_opt}, as shown in \eqref{admm_primal}.
Notice the correspondence between the diagonal coefficients of $\matE_{1}^{T}\matE_{1}$ and the number of bridge node connections per node.}
\label{fig:fig_lin_constraint_matrix}
\end{figure}
Using the newly defined terms, the optimization problem in (\ref{m_step_cbdsbl_mod}) can be rewritten compactly as
\begin{equation} \label{admm_primal}
\underset{\vgamma_{\J}, \vgamma_{\B}}{\text{min}}  f(\vgamma_{\J}) \hspace{0.5cm} \text{s.t. } \matE_{1}\vgamma_{\J} + \matE_{2}\vgamma_{\B} = 0
\end{equation}
where $f:\Real^{nL} \to \Real$ denotes the objective function in (\ref{m_step_cbdsbl_mod}), which depends only on $\vgamma_{\J}$. 
The augmented Lagrangian $L_{\rho}$ corresponding to \eqref{admm_primal} can also be rewritten compactly as 
\begin{eqnarray} \label{admm_augmented_lagrangian}
L_{\rho}(\vgamma_{\J}, \vgamma_{\B}, \vlambda) 
&=& 
f(\vgamma_{\J}) + \vlambda^{T}(\matE_{1}\vgamma_{\J} + \matE_{2}\vgamma_{\B}) 
\nonumber \\
&& \hspace{-1.8cm} \; + \; 
\displaystyle \frac{\rho}{2} (\matE_{1}\vgamma_{\J} + \matE_{2}\vgamma_{\B})^{T}(\matE_{1}\vgamma_{\J} + \matE_{2}\vgamma_{\B}).
\end{eqnarray}
By construction, the block matrix $\matE_{1}$ has full column rank, as all its columns are mutually disjoint in support. However $\matE_{1}$ can be row rank deficient due to repeated rows caused by a node being connected to multiple bridge nodes, which is often the case.
Since the matrix $\matE_{1}$ is row rank deficient, the ADMM convergence results of \cite{WeiDengWotaoYinADMM12} are not applicable to \eqref{admm_primal}. Theorem \ref{admm_convergence_theorem} below summarizes the convergence of the ADMM iterations  (\ref{admm_iteration3}), (\ref{vgammaj_update}) and (\ref{vgammab_update_simple})
to their fixed point. The result in Theorem~\ref{admm_convergence_theorem} holds for any $f$ that is strongly convex with strong convexity constant $m_{f}$, and with an $M_{f}$ Lipschitz continuous gradient. 
\begin{theorem} \label{admm_convergence_theorem}
Let $\{ \vgamma_{\J}^{*}$, $\vgamma_{\B}^{*} \}$ and $\vlambda^{*}$ denote the unique primal and dual optimal solutions of (\ref{admm_primal}),
and vector $\vecu$ be constructed as $\vecu = [(\matE_{2}\vgamma_{\B})^{T}\;\;\vlambda^{T}]^{T}$ (similarly for $\vecu^{r}, \vecu^{*}$). Then, it holds that
\begin{enumerate}
\item  The sequence $\vecu^{r}$ is Q-linearly\footnote{
A sequence ${x_{k}}: \mathcal{Z}_{+} \to \Real$ is said to be a \emph{Q-linearly} convergent to $L$, if there exists $\mu \in (0,1)$ such that
$ \underset{k \to \infty}{\lim} \frac{|x_{k+1} - L|}{|x_{k} - L|} = \mu $ \cite{WataoYin13ADMMConvergence}.  
}
convergent to $\vecu^{*}$, i.e., 
\begin{equation} \label{thm_vgammab_conv}
 \norm{\vecu^{r+1} - \vecu^{*}}_{\matG} \leq \frac{1}{1 + \delta} \norm{\vecu^{r} - \vecu^{*}}_{\matG}
\end{equation}
where $\delta$ is evaluated as 
\begin{equation} \label{defn_delta} 
\hspace{-0.75cm} \delta = \underset{\mu, \nu \geq 1}{\text{max}} 
\lc
\text{min}
\lb
\displaystyle 
\frac{2 m_{f}}{\frac{\nu M_{f}^{2}}{\rho(\nu - 1)\sigma_{\text{min}}^{2}} +
\mu \rho \sigma_{\text{max}}^{2}}  
, 
\frac{\sigma_{\text{min}}^{2}}{\nu \sigma_{\text{max}}^{2}}
,
\frac{\mu -1}{\mu} 
\rb \rc. 
\end{equation}
\item  The primal sequence $\vgamma_{\J}^{r}$ is R-linearly\footnote
{
A sequence ${x_{k}}:\mathcal{Z}_{+} \to \Real$ is said to be \emph{R-linearly} convergent to $L$, if there exists Q-linearly convergent sequence $y_{k}$ which converges to zero such that
$ \underset{k \to \infty}{\lim} {|x_{k} - L|} \leq y_{k}$.
}
convergent to $\vgamma_{\J}^{*}$, i.e.,
\begin{equation} \label{thm_vgammaj_conv}
 \norm{\vgamma_{\J}^{r+1} - \vgamma_{\J}{*}}_{2} \leq \frac{1}{2 m_{f}} \norm{\vecu^{r} - \vecu^{*}}_{\matG}
\end{equation}
\end{enumerate}
where $\norm{\cdot}_{\matG}$ is the weighted norm with respect to the diagonal matrix $\matG = \linebreak \diag{(\rho I_{n|\mathcal{B}|}, \rho^{-1}I_{N_{C}})}$.
\end{theorem}
\begin{proof}
See Appendix \ref{App:admm_convergence_proof}.
\end{proof}
According to Theorem \ref{admm_convergence_theorem}, the primal optimality gap $||\vgamma_{\J}^{r} - \vgamma_{\J}^{*}||_{2}$ decays R-linearly 
with each ADMM iteration. Moreover, since $\vgamma_{\J}^{*}$ is primal feasible, there is consensus among $\vgamma_{j}, j \in J$ upon convergence,
implying that each node effectively minimizes the centralized M-step cost function in (\ref{m_step_cbdsbl}).

\subsection{Selection of the Augmented Lagrangian Parameter $\rho$}\label{sec:rho_selection}
From \eqref{thm_vgammab_conv} and \eqref{thm_vgammaj_conv} in Theorem \ref{admm_convergence_theorem}, we observe that to optimize the decay of the primal optimality gap
between $\vgamma_{\J}^{r}$ and $\vgamma_{\J}^{*}$ in each ADMM iteration, the augmented Lagrangian parameter $\rho$ has to be chosen such that it maximizes $\delta$ 
in (\ref{defn_delta}). Theorem \ref{theorem_optimal_rho_n_delta} reveals the optimal value of $\rho$ and the corresponding value of $\delta$.

\begin{theorem} \label{theorem_optimal_rho_n_delta}
 The optimal value of augmented Lagrangian parameter $\rho$ which uniquely maximizes the $\delta$ as defined in \eqref{defn_delta} is given by 
\begin{equation} \label{rho_optimal}
 \rho_{\text{opt}} = \frac{M_{f}}{\sigma_{\text{max}}\sigma_{\text{min}}} 
 \ls
 \frac
 {\sqrt{(\kappa -1)^{2} + 4\kappa \kappa_{f}^{2} } + (\kappa - 1)}
 {\sqrt{(\kappa -1)^{2} + 4\kappa \kappa_{f}^{2} } - (\kappa - 1)}
 \rs^{\frac{1}{2}}.
\end{equation}
The corresponding maximal value of $\delta$ is given by 
\begin{equation} \label{delta_optimal}
\delta_{\text{opt}} = \frac{2}{\lb \kappa + 1 + \sqrt{(\kappa -1)^{2} + 4\kappa \kappa_{f}^{2} } \rb} 
\end{equation}
where $\kappa_{f} = \dfrac{M_{f}}{m_{f}}$ represents the condition number of the objective function in (\ref{m_step_cbdsbl_mod}) and 
$\kappa = \dfrac{\sigma_{\text{max}}^{2}}{\sigma_{\text{min}}^{2}}$ is the ratio of the maximum and minimum eigenvalues of $\matE_{1}^{T}\matE_{1}$.
\end{theorem}
\begin{proof}
See Appendix \ref{App:theorem_optimal_rho_n_delta_proof}.
\end{proof}

From (\ref{delta_optimal}), we observe that the convergence rate of the ADMM iteration in the M-step of CB-DSBL algorithm depends upon two factors: 
$\kappa$ and $\kappa_{f}$. $\kappa$ close to its minimum value of unity results in faster convergence of the ADMM iterations.
Since the ratio $\kappa = \dfrac{\sigma_{\text{max}}^{2}}{\sigma_{\text{min}}^{2}}$ is also equal to the ratio of maximum and minimum number of 
bridge nodes per node in the network, a rule of thumb for bridge node selection would be to ensure that each node is connected to more or less the same number of bridge nodes.
The convergence rate also depends upon $\kappa_{f}$, the parameter that is dependent on how well conditioned the function $f$ is. For the case where $f$ is the objective function in (\ref{m_step_cbdsbl_mod}), it is easy to show that $m_{f} = M_{f} = 2$ and $\kappa_{f} = 1$. Thus, specific to CB-DSBL, the optimal ADMM parameter $\rho$ is 
given by $\rho_{\text{opt}} = \frac{2}{\sigma_{\text{min}}^{2}}$ and the corresponding $\delta_{\text{opt}} = \frac{1}{\kappa + 1}$.
For a given network connectivity graph $\mathcal{G}$, this $\rho_{\text{opt}}$ can be computed off-line and programmed in each node. 
As shown in Fig.~\ref{fig_mse_niter_vs_rho_scalefactor}, the average MSE and mean number of iterations vary widely with $\rho$, 
an inappropriate choice of $\rho$ resulting in slow convergence and poor reconstruction performance. Also, the $\rho_{\text{opt}}$ computed in \eqref{rho_optimal}
is very close to the $\rho$ that results in both the fastest convergence as well as the lowest 
average MSE.
\begin{figure}
\centering
\begin{subfigure}[b]{0.24\textwidth}
    \includegraphics[width=\textwidth]{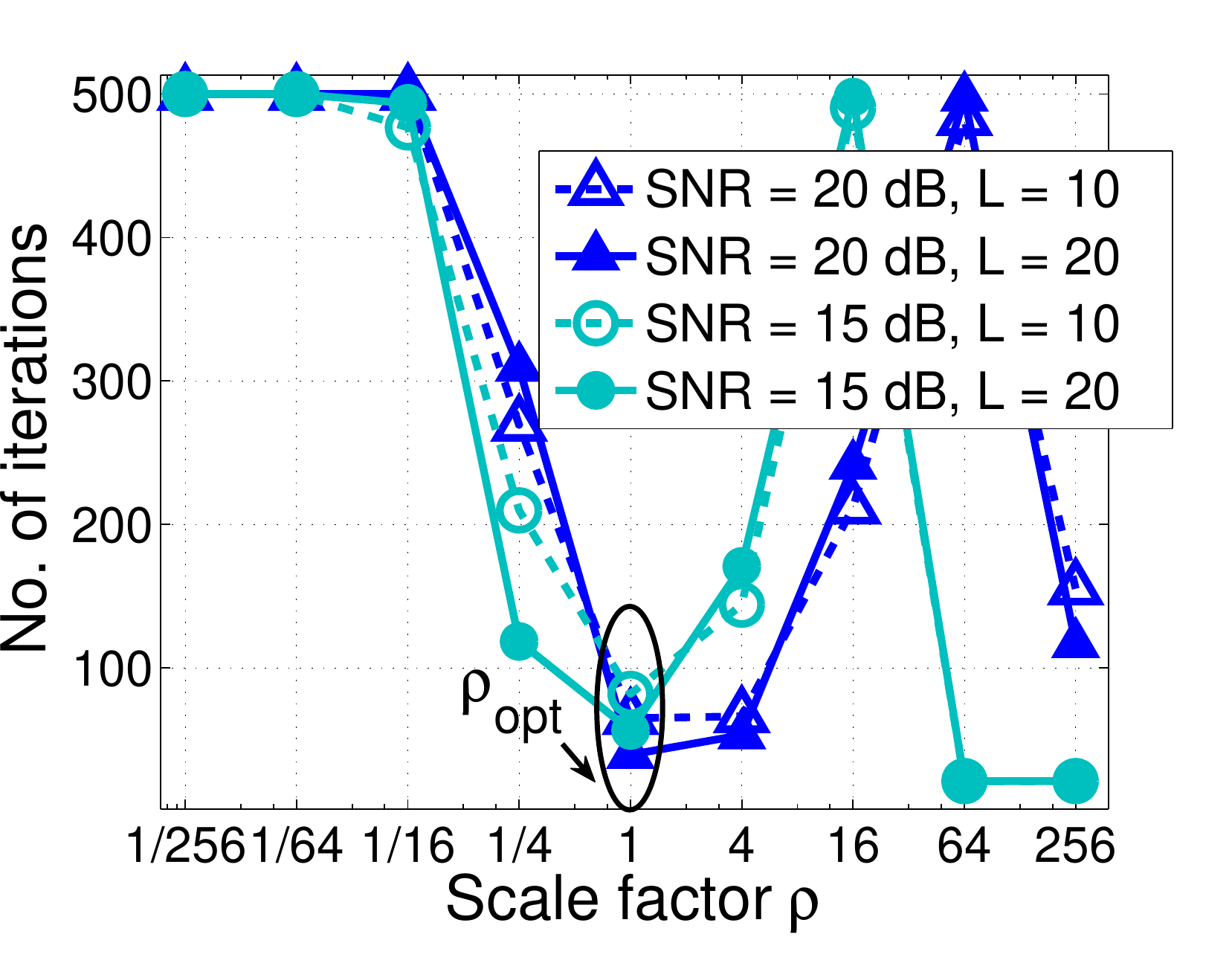}
\end{subfigure}%
\begin{subfigure}[b]{0.24\textwidth}
    \includegraphics[width=\textwidth]{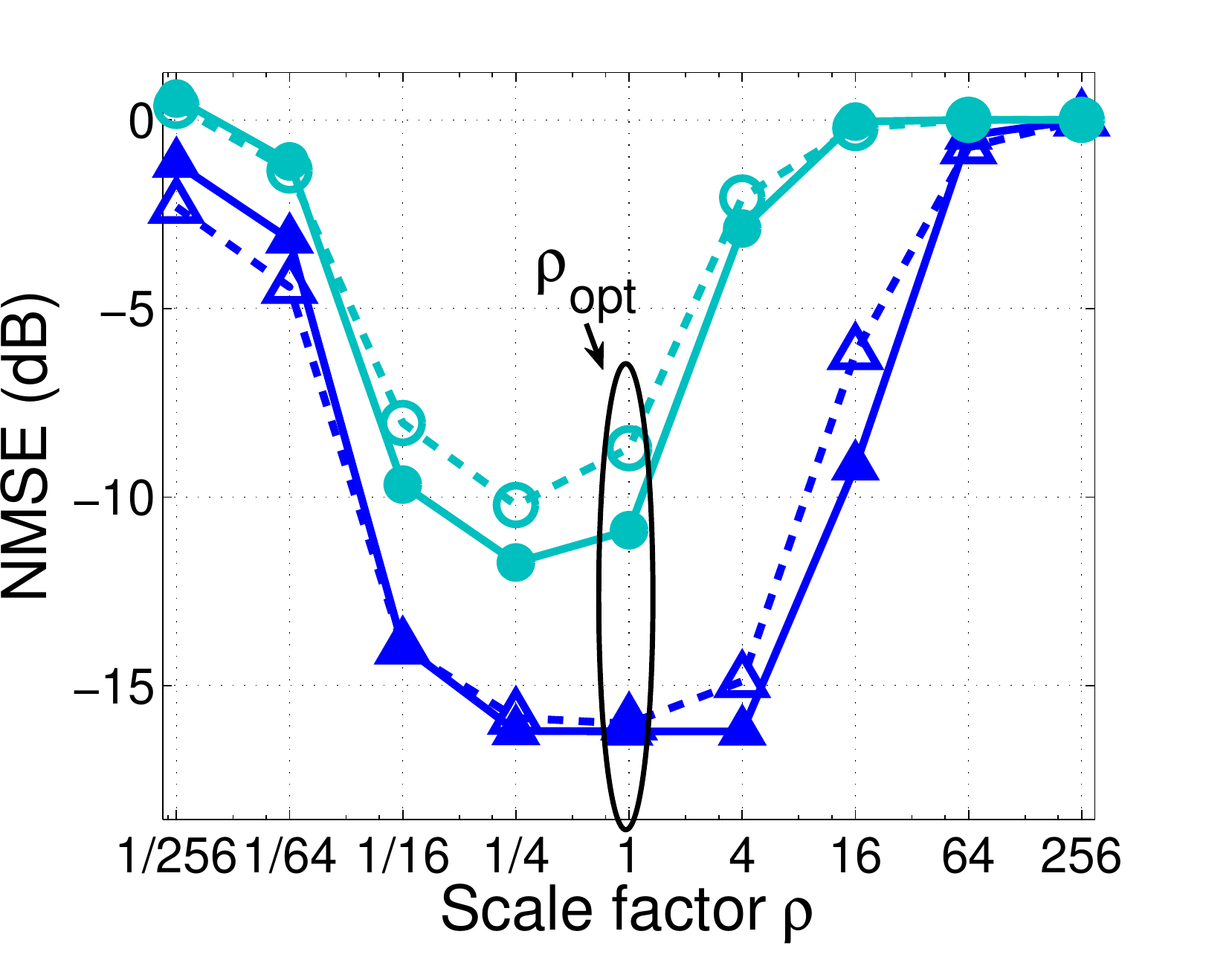}
\end{subfigure}
\caption{Left and right plots show the sensitivity of the number of iterations required for convergence and NMSE respectively with respect to the ADMM parameter $\rho$. 
The scale factor $\rho = 1$ corresponds to $\rho_{opt}$ in (\ref{rho_optimal}). }
\label{fig_mse_niter_vs_rho_scalefactor}
\end{figure}

\subsection{Computational Complexity of CB-DSBL}\label{sec:computational_complexity}
In this section, we discuss the computational complexity of the steps involved in a single iteration of the CB-DSBL algorithm. The local E-step requires $\mathcal{O}(n^2 + nm^2 + m^3)$ elementary operations at each node. The M-step is executed as multiple (say, $r_{\text{max}}$) ADMM iterations. A single ADMM iteration involves updating of the local hyperparameter estimate $\vgamma_{j}$ and Lagrange multipliers, which takes $\mathcal{O}(\zeta n)$ computations per node, $\zeta$ being the highest number of
bridge nodes assigned per node in the network. Further, each bridge node $b \in \B$ has to perform an additional $\mathcal{O}(\zeta n)$ computations to update the local bridge parameters $\vgamma_{b}$ in every ADMM iteration. Thus, the overall computational complexity of a single CB-DSBL algorithm at each node is $\mathcal{O}(n^2 + nm^2 + m^3 + \zeta n r_{\text{max}})$, and, as desired, it does not 
scale with $L$, i.e., the total number of nodes in the network.

\subsection{Other CB-DSBL Variants}\label{sec:cbdsbl_variants}
There are several alternatives to the aforementioned bridge node based ADMM technique that could potentially be used to solve the M-step optimization in \eqref{m_step_consensus_opt}. In this section, we present empirical results comparing the performance and communication complexity of four different variations of the proposed CB-DSBL algorithm
 based on (i) bridge node based ADMM \cite{Giannakis_08_NoisyLinks} (ii) Distributed ADMM (D-ADMM) \cite{Mota_13_ADMM} (iii) Consensus averaging Method of Multipliers (CA-MoM) \cite{HaoZhu09CAMoM},
 and (iv) EXact firsT ordeR Algorithm (EXTRA) \cite{WeiShi15EXTRA}. Each of these decentralized algorithms is endowed with at least $\mathcal{O}(\frac{1}{k})$ convergence rate, where $k$ stands for the
 iteration count. Besides these four, there are proximal gradient based methods \cite{Jakovetic14FastDistributedGradientMethods, ChenAndOzdaglar12FastDistProxGradMethod} relying on Nesterov-type
 acceleration techniques which also offer linear convergence rates. However, these algorithms require the objective function to be bounded and involve multiple communication rounds per iteration, 
 which is of major concern in our work. As shown in Fig. \ref{fig:inner_loop_variations}, the proposed CB-DSBL variant relying on the bridge node based ADMM scheme is the most communication efficient one. 
 
\subsection{Implementation Issues} \label{sec:cbdsbl_implementation}
CB-DSBL algorithm can be seen as a decentralized EM algorithm to find the ML estimate of the hyperparameters $\vgamma$ of a sparsity
inducing prior. CB-DSBL, not surprisingly, also inherits the tendency of the EM algorithm to converge to one of the multiple local maxima
of the ML cost function $\log{p(\matY | \vgamma)}$. However, getting trapped in a local maximum is not a problem, as it has been shown in \cite{Wipf_04_sbl}
that all local maxima of the $\log{p(\matY | \vgamma)}$ are at most $m$-sparse and hence qualify as reasonably good solutions to our original sparse model 
estimation problem. Despite this, it is recommended to seed the EM algorithm with $\vgamma$ whose all entries are close to zero. 

Another common issue is that of the wide variation in the energy of the nonzero entries of $\vecx_{j}$ across the network. 
Specifically, in distributed event classification by a multitude of different types of sensors \cite{Nasrabadi11MTMV}, each sensor node may employ its own 
distinct sensing modality and hence may perceive a different SNR. In such cases, a preconditioning step which normalizes the local response vector to unit energy
is recommended for fast convergence of the CB-DSBL algorithm. The local sparse signal estimates can be re-adjusted in the end to undo the pre-conditioning.  

\section{Simulation Results} \label{sec:sim_results}
In this section, we present simulation results to examine the performance and complexity aspects of the proposed CB-DSBL algorithm when compared with existing 
decentralized algorithms: DRL-1 \cite{Ling_13_jsm_lqnorm}, DCOMP \cite{Wimalajeewa13DCOMP} and DCSP \cite{Varshney14DCSP}. 
The centralized M-SBL \cite{Wipf_07_msbl} is also included in the study as a performance benchmark for the proposed decentralized algorithm. The CB-DSBL variant
considered here executes two ADMM iterations in the inner loop for every EM iteration in the outer loop. The value of the augmented Lagrangian parameter, $\rho$,
is chosen according to (\ref{rho_optimal}). For each experiment, the set $\B$ of bridge nodes is selected as described in section \ref{sec:cbdsbl_algorithm}. 
The local measurement matrices $\matPhi_{j}$ are chosen to be Gaussian random matrices with normalized columns. The nonzero signal coefficients are sampled 
independently from the Rademacher distribution, unless mentioned otherwise. For each trial, the connections between the nodes are assumed according to a randomly 
generated Erd{\"o}s-Renyi graph with a node connection probability of $0.8$. In the final step of M-SBL and CB-DSBL algorithms, the active support is identified
by element-wise thresholding the local hyperparameter vector $\vgamma_{j}$ at node $j$ using the threshold $4 \sigma_{j}^{2}$, where $\sigma_{j}^{2}$ denotes the
local measurement noise variance.

\subsection{Performance versus SNR} \label{sec:sim_results_perf_vs_SNR}
In the first set of experiments, we compare the normalized mean squared error (NMSE) and the normalized support error rate (NSER) of different algorithms 
for a range of SNRs. The support-aware LMMSE estimator sets the MSE performance benchmark for all the support agnostic algorithms considered here. 
The NMSE and NSER error metrics are defined as
\begin{equation}
  \text{NMSE} = \frac{1}{L}\sum_{j=1}^{L}\frac{||\vecx_{j} - \hat{\vecx}_{j}||_{2}^{2}}{||\vecx_{j}||_{2}^{2}}  \nonumber
\end{equation}
\begin{equation}
 \text{NSER} = \frac{1}{L}\sum_{j=1}^{L}\frac{|\mathcal{S} \backslash \hat{\mathcal{S}_{j}}| + |\hat{\mathcal{S}}_{j} \backslash \mathcal{S}|}{|\mathcal{S}|}  \nonumber
\end{equation}
where $\mathcal{S}$ is the true common support and $\hat{\mathcal{S}}_{j}$ is the support estimated at node $j$. 
The network size is fixed to $L = 10$ nodes. As seen in Fig.~\ref{fig:fig_compare_mse}, CB-DSBL matches the performance of centralized M-SBL in all cases.
For higher SNR ($\ge 15 $ dB), it can be seen that both M-SBL and proposed CB-DSBL are MSE optimal. CB-DSBL also outperforms 
DRL-1 and DCOMP in terms of both MSE and support recovery. This is attributed to the fact that the Gaussian prior used in CB-DSBL with its alternate interpretation 
as a variational approximation to the Student's t-distribution is more capable of inducing sparsity in comparison to the sum-log-sum penalty used in DRL-1. 
The poor performance of DCOMP is primarily due to its sequential approach towards  support recovery which prevents any corrections to be applied to the 
support estimate at each step of the algorithm. Contrary to \cite{Varshney14DCSP}, DCSP fails to perform better than DCOMP. This is because DCSP works only when the number of measurements exceeds $2k$, where $k$ is the size of the nonzero support.  
\begin{figure}[ht]
\begin{minipage}{.24\textwidth}
\centering
\includegraphics[width=\textwidth]{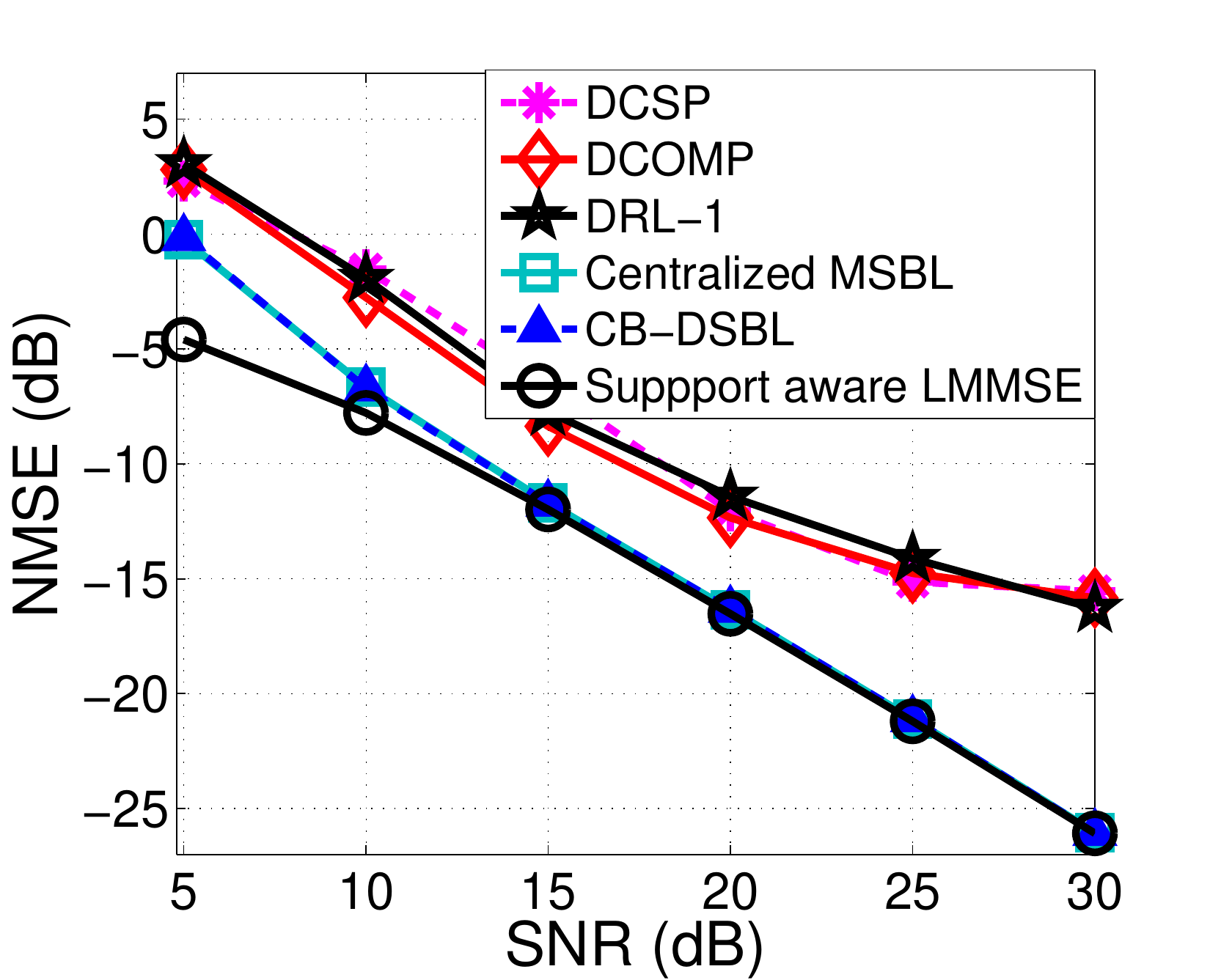}
\end{minipage}
\begin{minipage}{.24\textwidth}
\centering
\includegraphics[width=\textwidth]{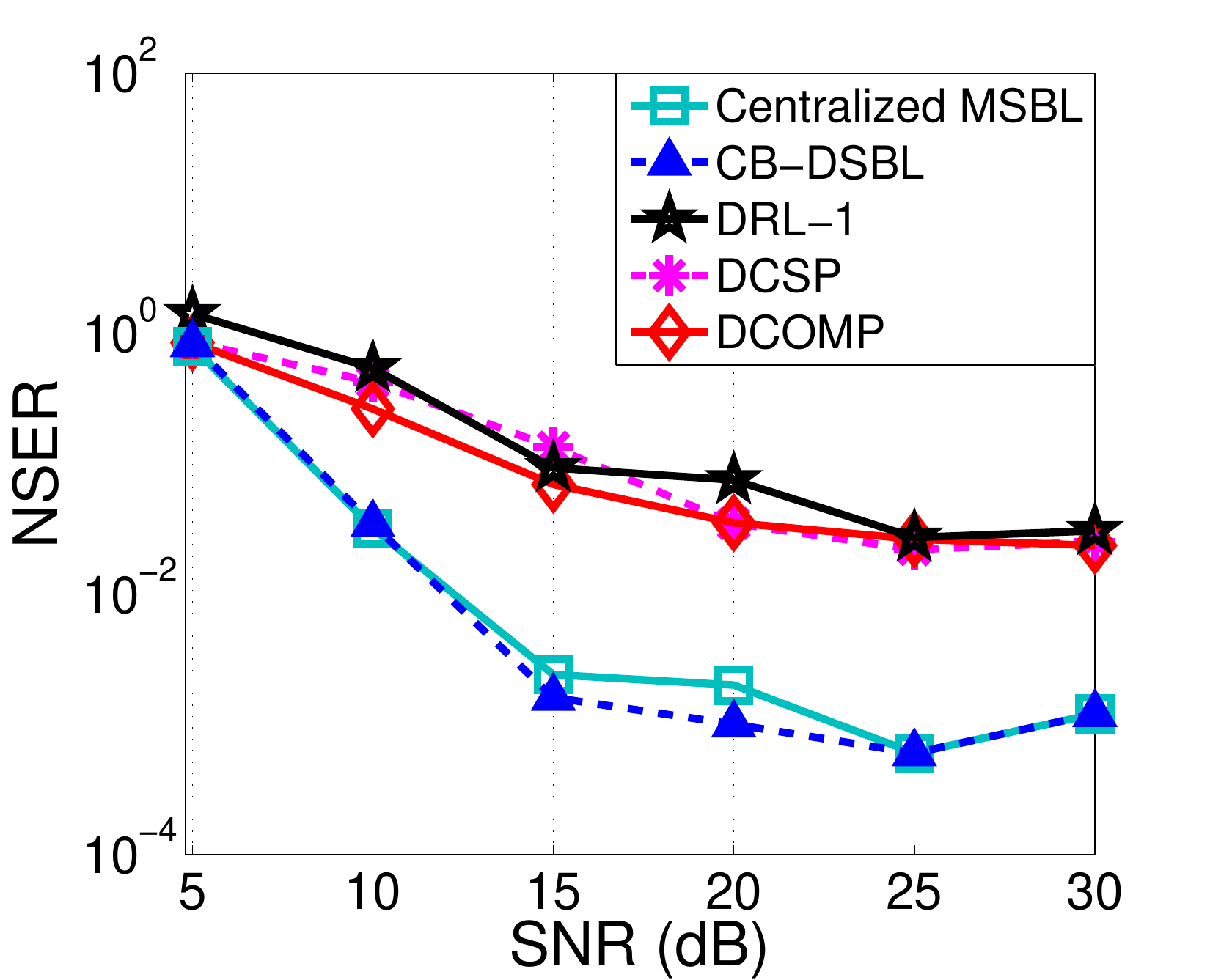}
\end{minipage}
\begin{subfigure}[b]{0.48\textwidth}
\vspace{0.1cm}
\caption{\scriptsize{Nonzero coefficients drawn from Rademacher distribution}}
\end{subfigure}

\begin{minipage}{.24\textwidth}
\centering
\includegraphics[width=\textwidth]{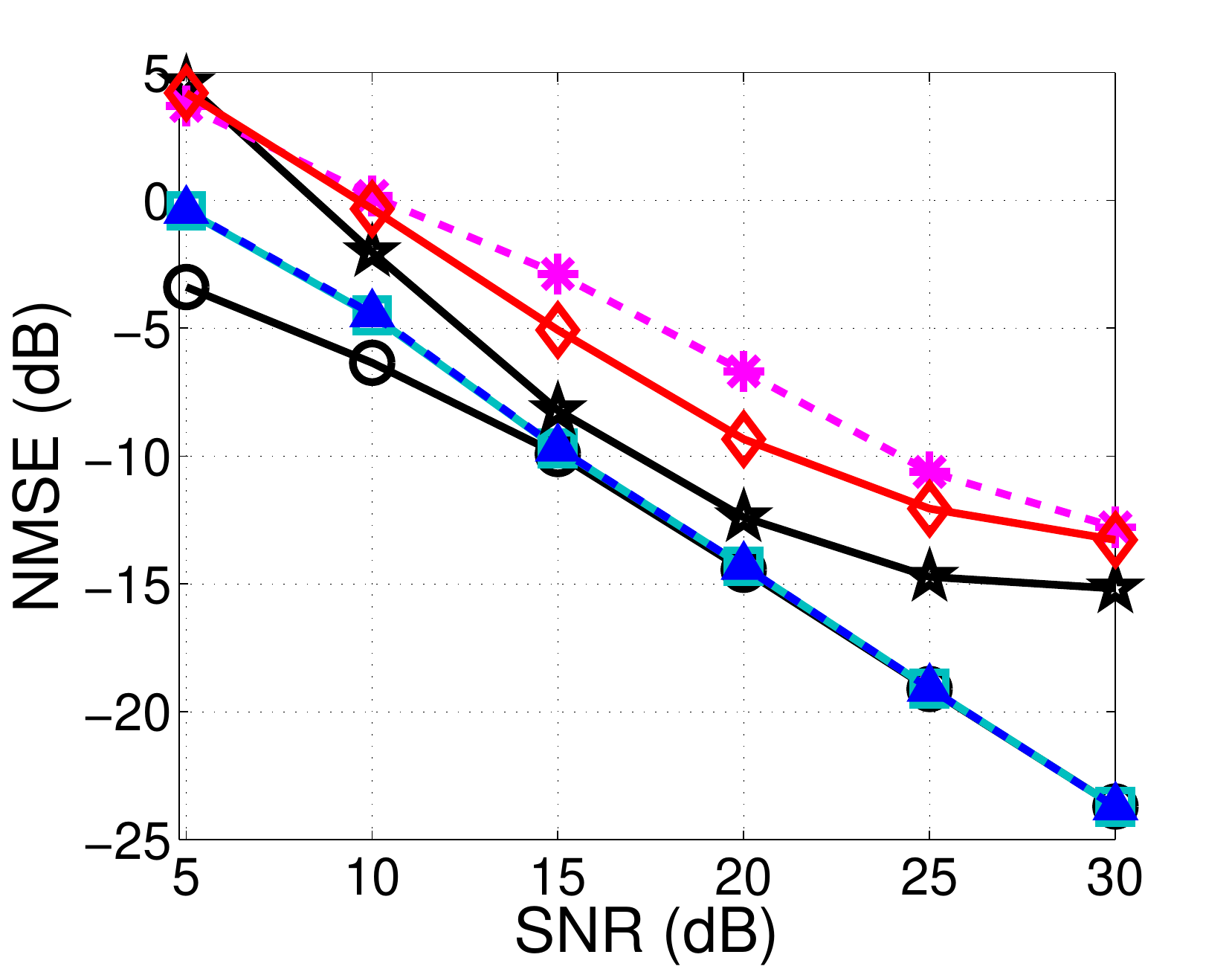}
\end{minipage}
\begin{minipage}{.24\textwidth}
\centering
\includegraphics[width=\textwidth]{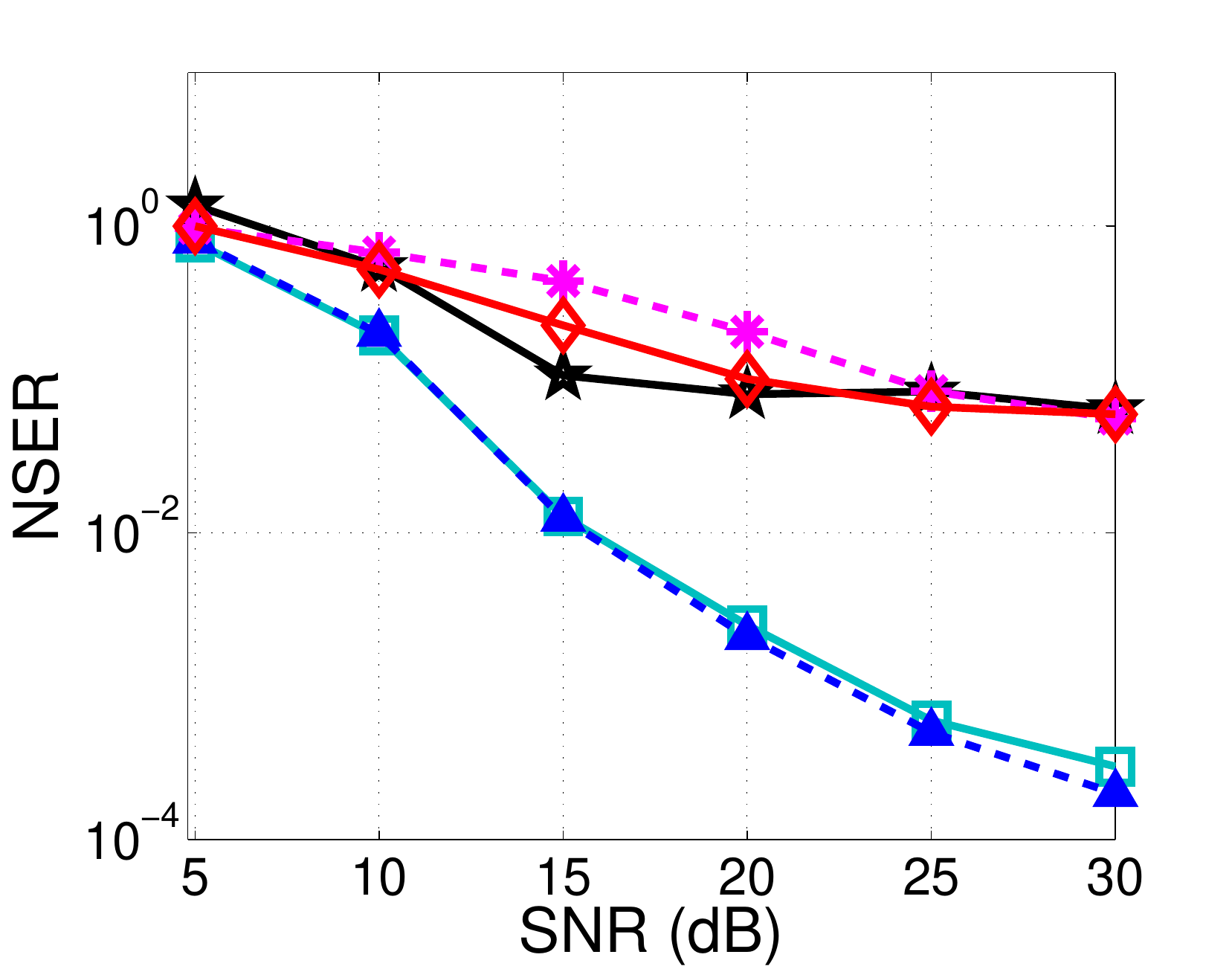}
\end{minipage}
\begin{subfigure}[b]{0.48\textwidth}
\vspace{0.1cm}
\caption{\scriptsize{Nonzero coefficients drawn from Gaussian distribution}}
\end{subfigure}
\caption{Left and right figures in the above plot the NMSE and NSER respectively for different SNRs. 
Other simulation parameters: $L = 10$ nodes, $n = 50$, $m = 10$ and $10 \%$ sparsity.}
\label{fig:fig_compare_mse}
\end{figure}

\subsection{Tradeoff between Measurement Rate and Network Size} \label{sec:sim_results_perf_vs_L}
In the second set of experiments, we characterize the NMSE phase transition of the different algorithms in $(m/n)-L$
 plane to identify the minimum measurement rate ($m/n$) needed to ensure less than $1\%$ signal reconstruction error 
 (or, NMSE $ \le -20$ dB), for different network sizes ($L$), and a fixed sparsity rate ($k/n = 0.1$). 
 As shown in Fig.~\ref{fig:fig_L_vs_M_tradeoff}, for the same network size, CB-DSBL is able to successfully 
 recover the unknown signals at a much lower measurement rate compared to DRL-1, DCOMP and DCSP. This plot 
 brings out the significant benefit of using collaboration between nodes and taking advantage of the JSM-2 model 
 in reducing the number of measurements required per node for successful signal recovery. Additionally, as the 
 network grows in size, the complexity of the local computations at each node also reduces with the number 
 of local measurements (see section~\ref{sec:computational_complexity}).
\begin{figure}[h!t]
\centering
\includegraphics[width=0.42\textwidth]{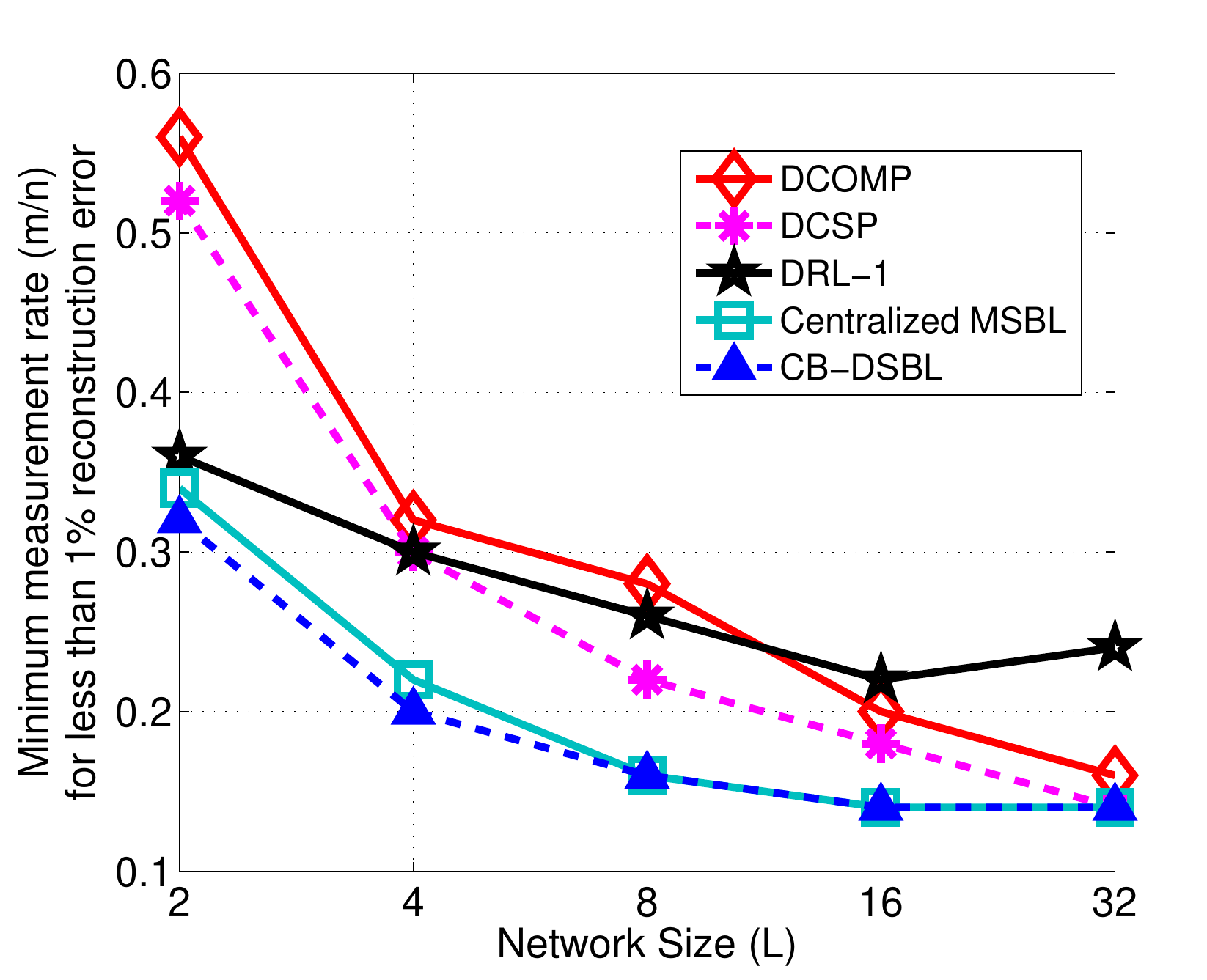}
\caption{NMSE phase transition plots of different algorithms illustrating the dependence of minimum measurement rate required to guarantee less than $1\%$ signal reconstruction error on the network size, for signal sparsity rate fixed at $10 \%$. Other simulation parameters: $n = 50$ and SNR = $30$ dB.}
\label{fig:fig_L_vs_M_tradeoff}
\end{figure}

\subsection{Performance versus Measurement Rate ($\frac{m}{n}$)} \label{sec:sim_results_perf_vs_undersamprate}
In the third set of experiments, we compare the algorithms with respect to their ability to recover the exact support for different undersampling ratios. 
As seen in Fig.~\ref{fig_prob_supp_recovery}, for a similar network size, CB-DSBL is able to exploit the joint sparsity structure better than DCOMP, DCSP and DRL-1,
and can correctly recover the support from significantly fewer number of measurements per node. Once again, CB-DSBL has identical support recovery performance as the centralized M-SBL, which was one of our design goals. 
\begin{figure}[h!]
\centering
\includegraphics[width=0.45\textwidth]{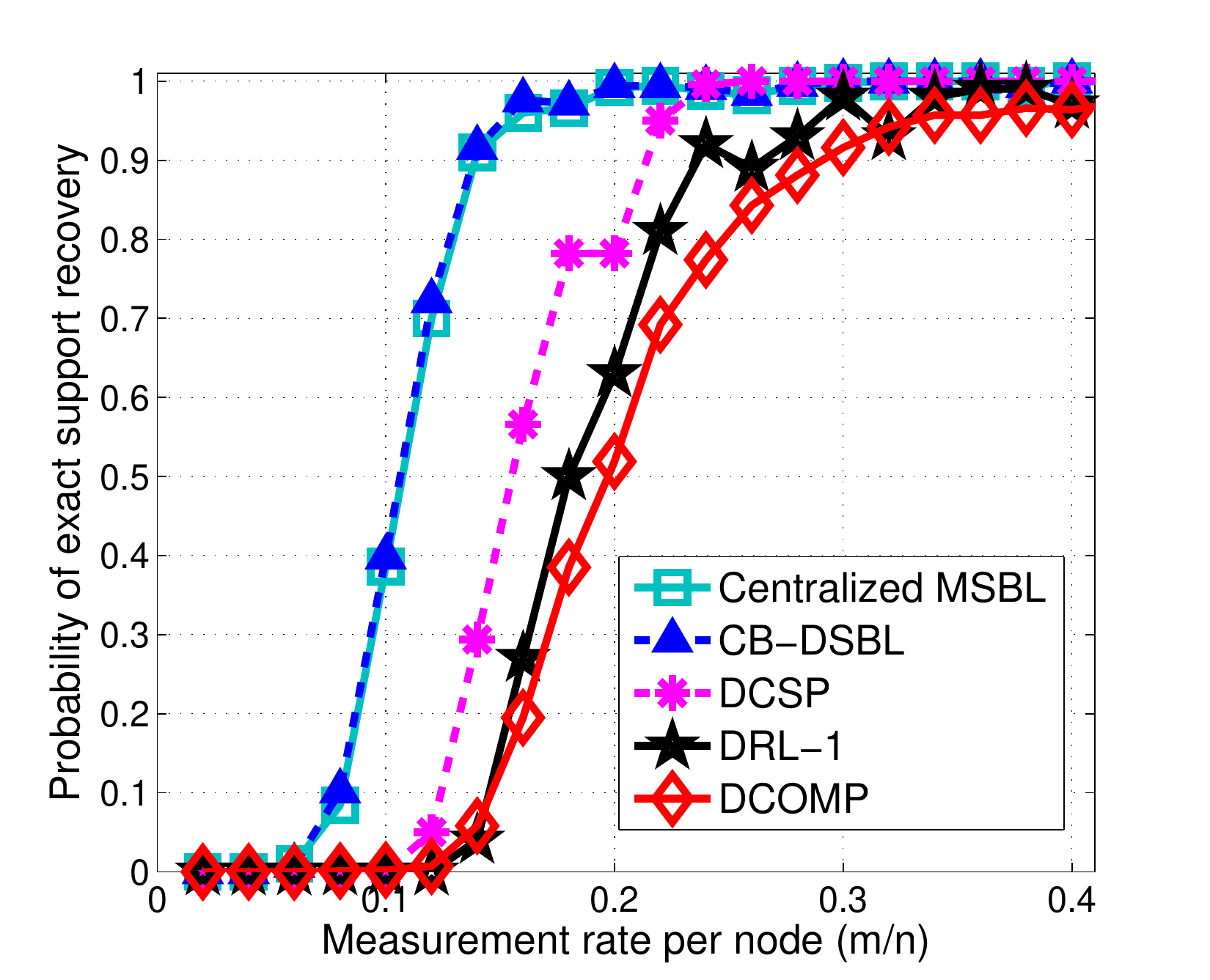}
\caption{Probability of exact support recovery versus number of measurements. Simulation parameters: $n = 50$, $10 \%$ sparsity, SNR = $15$ dB and $L = 10$ nodes.}
\label{fig_prob_supp_recovery}
\end{figure}

\subsection{Phase Transition Characteristics} \label{sec:sim_phase_transition}
In these set of experiments, we compare the phase transition behavior of different algorithms under NMSE and support recovery based pass/fail criteria.
Fig.~\ref{fig:fig_mse_phtr} plots the MSE phase transition of different algorithms where any point below the phase transition curve 
represents a sparsity rate $(k/n)$ and measurement rate $(m/n)$ tuple which results in an NMSE smaller than $-20$~dB corresponding to smaller than $1$~percent signal reconstruction error. Likewise, in Fig.~\ref{fig:fig_supp_phtr}, points below the support recovery phase transition curve represent $(k/n, m/n)$ tuples 
which result in more than $90$ percent accurate nonzero support reconstruction across all the nodes. Again, we see that the CB-DSBL and centralized M-SBL have identical performance and both are capable of signal reconstruction from considerably fewer measurements compared to DRL-1, DCOMP and DCSP. 
\begin{figure}[h!]
\centering
\begin{subfigure}[b]{0.42\textwidth}
\includegraphics[width=\textwidth]{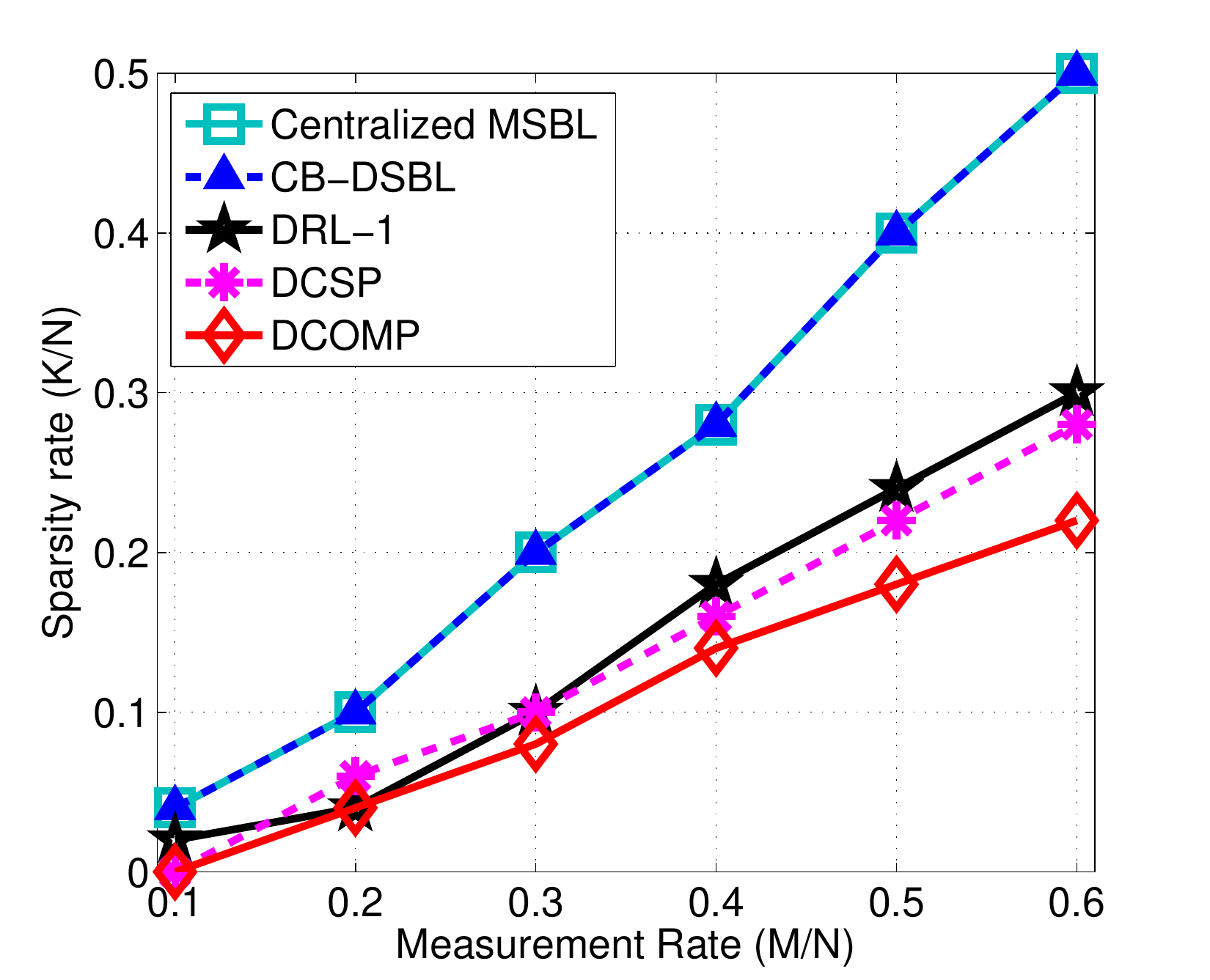}
\caption{NMSE phase transition}
\label{fig:fig_mse_phtr}
\end{subfigure}%

\begin{subfigure}[b]{0.42\textwidth}
\includegraphics[width=\textwidth]{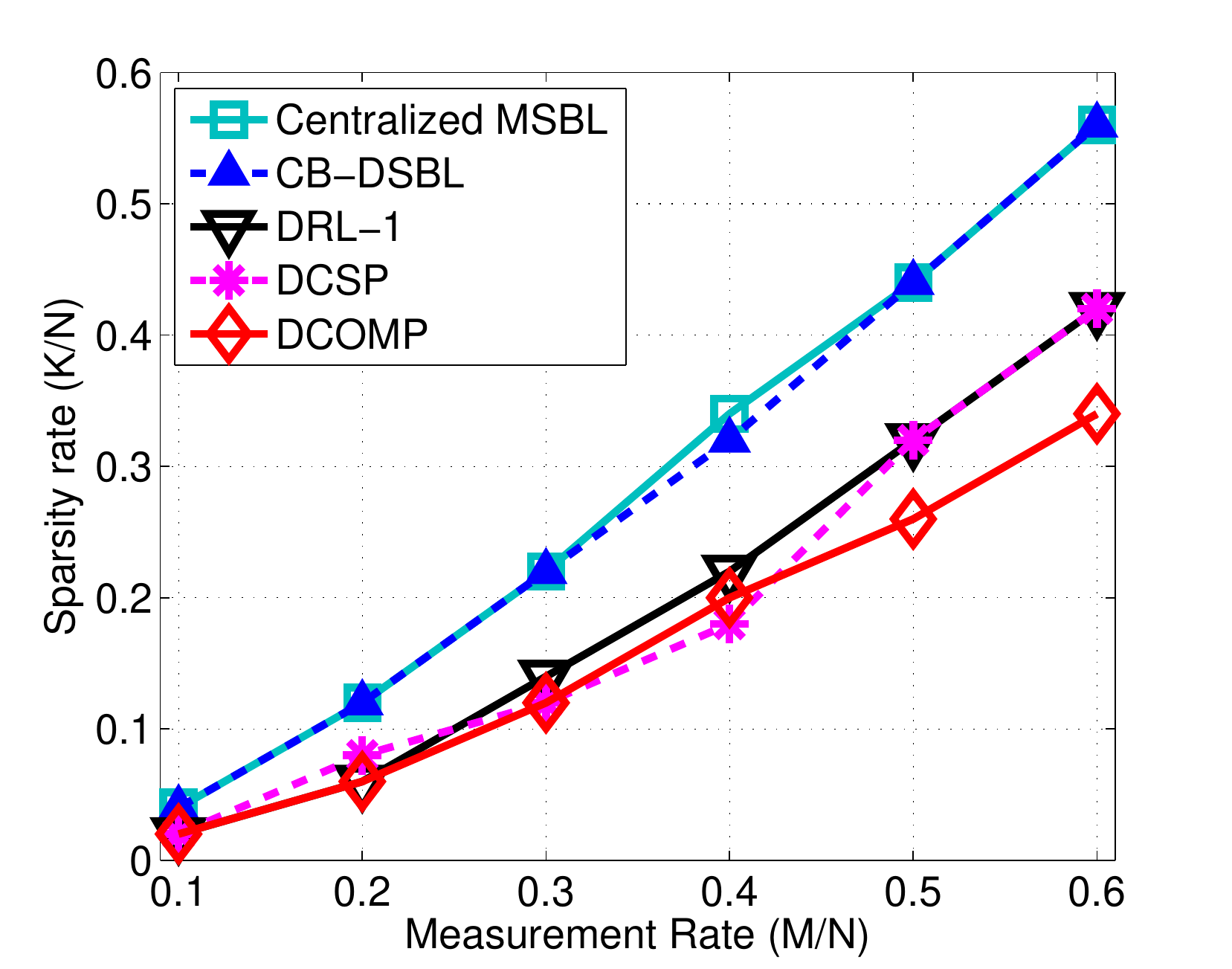}\caption{Support recovery phase transition}
\label{fig:fig_supp_phtr}
\end{subfigure}
\caption{Phase transition plots for the different joint-sparse signal recovery algorithms. For all points on or below the NMSE phase transition curve, at most
$1\%$ average signal reconstruction error is incurred by the respective algorithm. Likewise, for all points on or below the support recovery phase transition curve,
at least $90\%$ of the nonzero support is successfully identified at all the network nodes. Other simulation parameters: $n = 50$, $L=5$ nodes, SNR = $30$ dB and
number of trials = $200$.  
}\label{fig:fig_phtr}
\end{figure}

\subsection{Tradeoff between Number of Bridge Nodes and Robustness to Node Failures} \label{sec:bridgeNode_robustness_tradeoff}
In the final set of experiments, we demonstrate empirically that increasing the number of bridge nodes 
in the CB-DSBL algorithm makes it more robust to random node failures. As shown in Fig. \ref{fig:fig_bridge_node_regress}, 
by gradually increasing the density of bridge nodes in the network,  the CB-DSBL algorithm is able to tolerate higher 
rates of node failures without compromising on signal reconstruction performance. More interestingly, 
only a relatively small fraction of nodes need to be 
bridge nodes ($< 10 \%$ of the total network size) to ensure that CB-DSBL operates robustly in the face of random node failures.
\begin{figure}[ht]
\centering
\includegraphics[width = 0.44\textwidth, height = 0.35\textwidth]{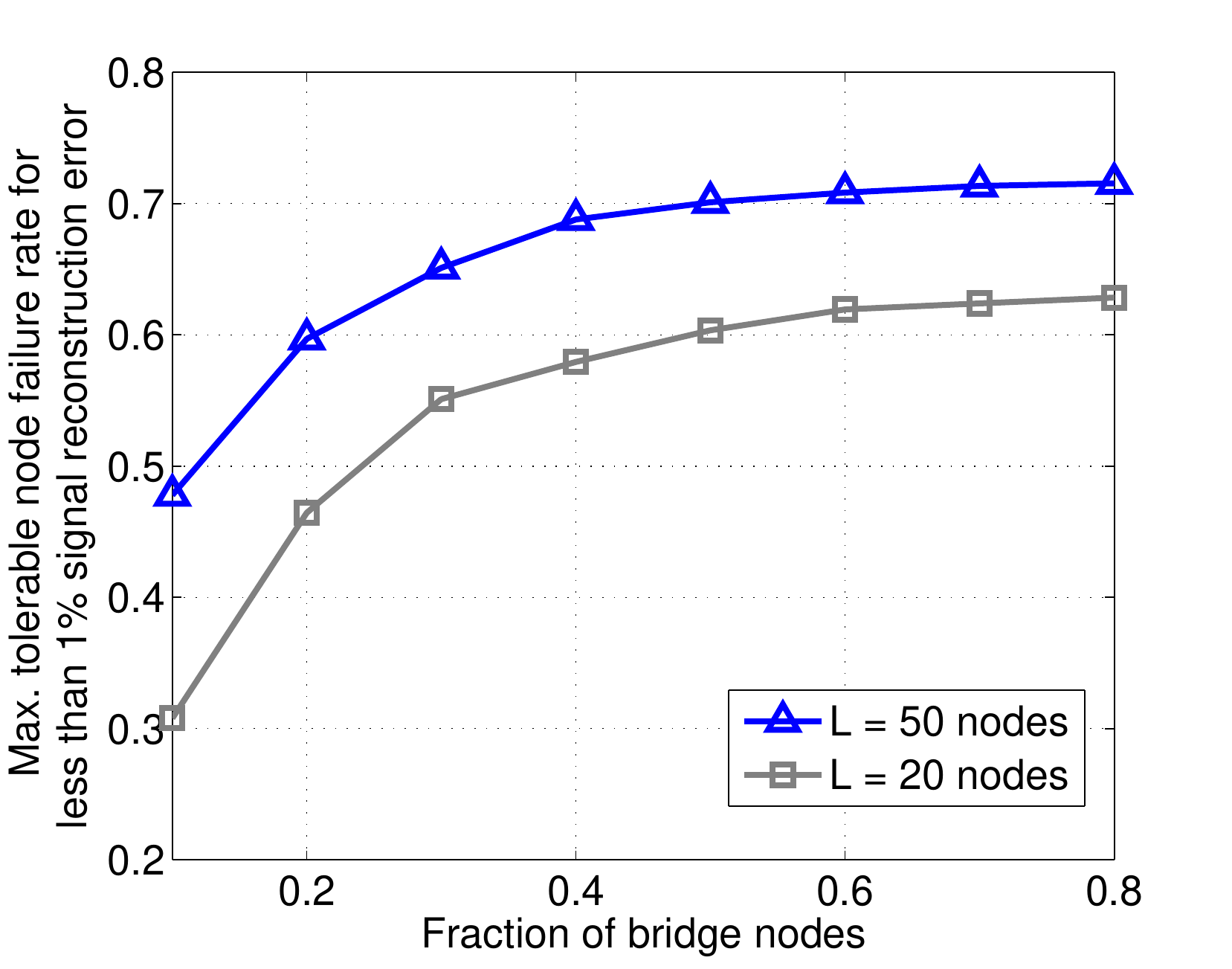}
\caption{Plot illustrating the trade off between the density of bridge nodes and the robustness of the proposed CB-DSBL algorithm to random node failures. For a given fraction of bridge nodes (no. of bridge nodes / $L$), each point on the curve represents the average node failure rate that can be tolerated by CB-DSBL while still achieving less than $1\%$ signal reconstruction error ($<-20$ dB NMSE). The higher the number of bridge nodes, the more tolerant the network is to random node failures.}
\label{fig:fig_bridge_node_regress}
\end{figure}

\section{Conclusions} \label{sec:conclusions}
In this paper, we proposed a novel iterative Bayesian algorithm called CB-DSBL for decentralized estimation of joint-sparse signals by multiple nodes in a network. 
The CB-DSBL algorithm employs ADMM based decentralized EM procedure to efficiently learn the parameters of a joint sparsity inducing signal prior which is shared 
by all the nodes, and is subsequently used in the MAP estimation of the local signals. The CB-DSBL algorithm is well suited for applications where the privacy of
the signal coefficients is important, as there is no direct exchange of either measurements or signal coefficients between the nodes. Experimental results showed 
that CB-DSBL outperforms existing decentralized algorithms: DRL-1, DCOMP and DCSP, in terms of both NMSE as well as support recovery performance. We also established
R-linear convergence of the underlying decentralized ADMM iterations. The amount of inter-node communication during the ADMM iterations is controlled by restricting each node to
exchange information with only a small subset of its single hop neighbors. For this inter-node communication scheme the ADMM convergence results presented here are
applicable to any consensus driven optimization of a convex objective function. Future extensions of this work could encompass exploiting any inter vector correlation
between the jointly sparse signals. Also, it would be interesting to analyze the convergence of CB-DSBL algorithm in the presence of noisy communication links between
nodes and under asynchronous network operation.

\section*{Appendix}
\renewcommand{\thesubsection}{\Alph{subsection}}

\subsection{Derivation of the M-step Cost Function}\label{App:appendix_mstep_cf}
The conditional expectation in (\ref{mstep_theory}) can be simplified as shown below.
\begingroup
\allowdisplaybreaks
\begin{align}
& \mathbb{E}_{\matX} \ls \log{p(\matY, \matX ; \vgamma)} | \matY; \vgamma^{k} \rs  
\nonumber \\
& =
\mathbb{E}_{[\matX | \matY; \vgamma^{k}]} [ \log{p(\matY | \matX)} +  \log{p(\matX ; \vgamma)}]  \nonumber \\
\label{estep_simplify_0}
& =
\mathbb{E}_{[\matX | \matY; \vgamma^{k}]} \log{p(\matY | \matX)}  +  \sum_{j \in \J} \mathbb{E}_{[\vecx_{j} | \vecy_{j}; \vgamma^{k}]} \log{p(\vecx_{j} ; \vgamma)}.
\end{align}
\endgroup
Using (\ref{signal_prior_for_x}), and discarding the terms independent of $\vgamma$ in (\ref{estep_simplify_0}), the M-step objective function 
$Q(\vgamma | \vgamma^{k})$ is given by
\begingroup
\allowdisplaybreaks
\begin{eqnarray} 
Q(\vgamma | \vgamma^{k}) 
&=& 
\sum_{j \in \J}\mathbb{E}_{[\vecx_{j} | \vecy_{j}, \vgamma^{k}]} \; \left( -\frac{1}{2}\log{|\mathbf{\Gamma}|} - \frac{1}{2} \vecx_{j}^{T}\mathbf{\Gamma}^{-1} \vecx_{j} \right)  \nonumber \\
&& \hspace{-2cm} =
-\frac{1}{2} \sum_{j \in \J} 
\lb
\log{|\mathbf{\Gamma}|} 
+ 
\sum_{i = 1}^{n} 
 \frac{ \mathbb{E}_{[\vecx_{j} \sim \mathcal{N}(\vmu_{j}^{k+1}, \matSigma_{j}^{k+1})]}\vecx_{j}(i)^{2} }{\vgamma(i)} 
\rb
\nonumber \\
&& \hspace{-2cm} =
-\frac{1}{2}
\sum_{j \in \J}\sum_{i = 1}^{n} 
\lb
\log{\vgamma(i)} + \frac{\matSigma_{j}^{k+1}(i,i) + \vmu_{j}^{k+1}(i)^{2}}{\vgamma(i)}  
\rb.
\end{eqnarray}
\endgroup

\subsection{Derivation of the Simplified Update for $\vgammab$}\label{App:simplify_admm_iters}
By summing the dual variable update rule (\ref{admm_iteration3}) across all nodes, the following holds for all $b \in \B$
\begin{eqnarray} \label{admm_iter3_summed}
\sum_{j \in \N_{b}}(\vlambdajb)^{r+1} = \sum_{j \in \N_{b}}(\vlambdajb)^{r} + \rho \sum_{j \in \N_{b}}\vgammaj^{r+1} - \rho |\N_{b}| \vgammab^{r+1}.
\end{eqnarray}
Plugging (\ref{vgammab_update}) in (\ref{admm_iter3_summed}), we obtain
\begin{equation} \label{lagrangianmult_locavg_zero}
 \sum_{j \in \N_{b}} (\vlambdajb)^{r+1} = 0 
 \hspace{1cm} \forall \; b \in \B.
\end{equation}
Using (\ref{lagrangianmult_locavg_zero}) in (\ref{vgammab_update}), we obtain the simplified update for~$\vgamma_{b}$.

\subsection{Proof of Theorem 1} \label{App:admm_convergence_proof}
The proof of the convergence of ADMM discussed in the sequel is a based on the proof given in \cite{WataoYin13ADMMConvergence}. However, our proof differs from the one in \cite{WataoYin13ADMMConvergence} due to the different scheme adopted here, which uses the auxiliary/bridge nodes to enforce consensus between the nodes. We make the following assumptions about the objective function $f$ in (\ref{admm_primal}).
\begin{enumerate}
 \item $f$ is twice differentiable and strongly convex in $\vgamma_{\J}$. This implies that there exists $m_{f} \in \Real_{+} \backslash \{0\}$ such that,
 for all $\vgamma_{\J}, \vgamma_{\J}^{'}$, the following holds
\begin{equation} \label{strongly_convex_f}
\langle \nabla f(\vgamma_{\J})^{T} - \nabla f(\vgamma_{\J}^{'})^{T},
\vgamma_{\J} - \vgamma_{\J}^{'} \rangle 
\;\; \geq \;\;
m_{f} ||\vgamma_{\J} - \vgamma_{\J}^{'}||_{2}^{2}.  
\end{equation}
\item $\nabla f$ is Lipschitz continuous, i.e., there exists a positive scalar $M_{f}$ such that, for all $\vgamma_{\J}, \vgamma_{\J}^{'}$, we have 
\begin{equation} \label{lipschitz_continuity_for_gradient}
||\nabla f(\vgamma_{\J}) - \nabla f(\vgamma_{\J}^{'})||_{2} \leq M_{f} ||\vgamma_{\J} - \vgamma_{\J}^{'}||_{2}.
\end{equation}
\end{enumerate}
Let $r$ denote the ADMM iteration count. From the zero subgradient optimality conditions corresponding to (\ref{admm_iteration1}) and (\ref{admm_iteration2}), we have
\begin{equation} \label{admm_eq1}
\nabla f(\vgamma_{\J}^{r+1})^{T} 
+ \matE_{1}^{T}\vlambda^{r} 
+ \rho \matE_{1}^{T}\matE_{1}\vgamma_{\J}^{r+1} 
+ \rho \matE_{1}^{T}\matE_{2}\vgamma_{\B}^{r}
= 0
\end{equation} 
\begin{equation} \label{admm_eq2}
\matE_{2}^{T}\vlambda^{r} 
+ \rho \matE_{2}^{T}\matE_{2}\vgamma_{\B}^{r+1} 
+ \rho \matE_{2}^{T}\matE_{1}\vgamma_{\J}^{r+1}
= 0.
\end{equation} 
From the dual variable update equation, we have,
\begin{equation} \label{admm_eq3}
\vlambda^{r+1} = \vlambda^{r} + \rho(\matE_{1}\vgamma_{\J}^{r+1} + \matE_{2} \vgamma_{\B}^{r+1}).  
\end{equation} 
Premultiplying (\ref{admm_eq3}) with $\matE_{1}^{T}$ and $\matE_{2}^{T}$ followed by its summation to (\ref{admm_eq1}) and (\ref{admm_eq2}) respectively gives 
\begin{equation} \label{admm_diff_eq1}
\nabla f(\vgamma_{\J}^{r+1})^{T} 
+ \matE_{1}^{T}\vlambda^{r+1} 
+ \rho \matE_{1}^{T}\matE_{2}(\vgamma_{\B}^{r} - \vgamma_{\B}^{r+1}) 
= 0.
\end{equation} 
\begin{equation} \label{admm_diff_eq2}
\matE_{2}^{T}\vlambda^{r+1} = 0.
\end{equation} 
By initializing $\vlambda$ equal to zero, $\vlambda^{r}$ always lies in the nullspace $\mathcal{N}(\matE_{2}^{T})$, physically implying that the sum of the Lagrange multipliers of nodes connected to a given bridge node is always equal to zero.
Let us assume $\vgamma_{\J}^{r} \rightarrow \vgamma_{\J}^{*}$, $\vgamma_{\B}^{r} \rightarrow \vgamma_{\B}^{*}$ and 
$\vlambda^{r} \rightarrow \vlambda^{*}$ as $r \rightarrow \infty$, then putting $r \rightarrow \infty$ in (\ref{admm_eq3}), (\ref{admm_diff_eq1}) and (\ref{admm_diff_eq2})
gives
\begin{equation} \label{admm_conv_eq1}
\nabla f(\vgamma_{\J}^{*})^{T} 
+ \matE_{1}^{T}\vlambda^{*} 
= 0
\end{equation} 
\begin{equation} \label{admm_conv_eq2}
\matE_{2}^{T}\vlambda^{*} = 0
\end{equation} 
\begin{equation} \label{admm_conv_eq3}
\matE_{1}\vgamma_{\J}^{*} + \matE_{2}\vgamma_{\B}^{*} = 0.  
\end{equation} 
Note that the condition (\ref{admm_conv_eq3}) implies consensus among $\vgamma_{j}, j \in \J$, upon convergence. By subtracting (\ref{admm_conv_eq1}),
(\ref{admm_conv_eq2}) and (\ref{admm_conv_eq3}) from (\ref{admm_diff_eq1}), (\ref{admm_diff_eq2}) and (\ref{admm_eq3}), respectively, we get the desired
difference terms needed for showing convergence results.
\begin{eqnarray} \label{diff_term1}
 \nabla f(\vgamma_{\J}^{r+1})^{T} - \nabla f(\vgamma_{\J}^{*})^{T}
+ \matE_{1}^{T}(\vlambda^{r+1} - \vlambda^{*}) 	 
\nonumber \\
 + \rho \matE_{1}^{T}\matE_{2}(\vgamma_{\B}^{r} - \vgamma_{\B}^{r+1}) = 0    
\end{eqnarray}
\begin{equation} \label{diff_term2}
\matE_{2}^{T}(\vlambda^{r+1} - \vlambda^{*}) = 0
\end{equation}
\begin{equation} \label{diff_term3}
\vlambda^{r+1} - \vlambda^{r} = 
\rho \matE_{1}(\vgamma_{\J}^{r+1} - \vgamma_{\J}^{*}) 
+ \rho \matE_{2}(\vgamma_{\B}^{r+1} - \vgamma_{\B}^{*}).
\end{equation}
Premultiplying (\ref{diff_term3}) with $\matE_{2}^{T}$ and using (\ref{admm_diff_eq2}), we obtain,
\begin{equation} \label{diff_term4}
\matE_{2}^{T} \matE_{1}(\vgamma_{\J}^{r+1} - \vgamma_{\J}^{*}) 
= - \matE_{2}^{T} \matE_{2}(\vgamma_{\B}^{r+1} - \vgamma_{\B}^{*}).
\end{equation}
Further from strong convexity of $f$ and using (\ref{diff_term1}), we can write,
\begin{flalign}
& m_{f} ||\vgamma_{\J}^{r+1} - \vgamma_{\J}^{*}||_{2}^{2} 
\leq
\langle \matE_{1}^{T} (\vlambda^{*} - \vlambda^{r+1}),
\vgamma_{\J}^{r+1} - \vgamma_{\J}^{*} \rangle 
\nonumber \\
& \hspace{1.5cm}
\;+\; 
\rho \langle \matE_{1}^{T} \matE_{2} (\vgamma_{\B}^{r+1} - \vgamma_{\B}^{r}),
(\vgamma_{\J}^{r+1} - \vgamma_{\J}^{*}) \rangle 
\displaybreak[0] \nonumber \\ 
& = 
\langle (\vlambda^{*} - \vlambda^{r+1}),
\matE_{1} (\vgamma_{\J}^{r+1} - \vgamma_{\J}^{*}) \rangle 
\nonumber \\
& \hspace{1.5cm}
\;+\; 
\rho \langle (\vgamma_{\B}^{r+1} - \vgamma_{\B}^{r}),
\matE_{2}^{T} \matE_{1}(\vgamma_{\J}^{r+1} - \vgamma_{\J}^{*}) \rangle 
\displaybreak[0] \nonumber  \\
& = 
\langle (\vlambda^{*} - \vlambda^{r+1}),
\matE_{1} (\vgamma_{\J}^{r+1} - \vgamma_{\J}^{*}) \rangle 
\nonumber \\
& \hspace{1.5cm}
\;-\; 
\rho \langle (\vgamma_{\B}^{r+1} - \vgamma_{\B}^{r}),
\matE_{2}^{T} \matE_{2}(\vgamma_{\B}^{r+1} - \vgamma_{\B}^{*}) \rangle 
\displaybreak[0] \nonumber  \\
& = 
\langle (\vlambda^{*} - \vlambda^{r+1}),
\frac{1}{\rho}(\vlambda^{r+1} - \vlambda^{r})
\;-\; 
\matE_{2} (\vgamma_{\B}^{r+1} - \vgamma_{\B}^{*}) \rangle  
\nonumber \\
& \hspace{1.5cm}
\;-\; 
\rho \langle (\vgamma_{\B}^{r+1} - \vgamma_{\B}^{r}),
\matE_{2}^{T} \matE_{2}(\vgamma_{\B}^{r+1} - \vgamma_{\B}^{*}) \rangle 
\displaybreak[0] \nonumber  \\
 & = 
\frac{1}{\rho} \langle (\vlambda^{*} - \vlambda^{r+1}),
(\vlambda^{r+1} - \vlambda^{r}) \rangle
\nonumber \\
& \hspace{1.5cm}
\;+\; 
\rho \langle \matE_{2}(\vgamma_{\B}^{r+1} - \vgamma_{\B}^{r}),
\matE_{2}(\vgamma_{\B}^{*} - \vgamma_{\B}^{r+1}) \rangle. 
\label{vgammaj_matnorm_ineq1}
\end{flalign}  
Here, the first identity is obtained by using a property of the inner product. The second, third and fourth identities are 
obtained by using (\ref{diff_term4}), (\ref{diff_term3}) and (\ref{diff_term2}) respectively.  
By defining $\linebreak \vecu = [(\matE_{2}\vgamma_{\B})^{T} \; | \; \vlambda^{T} ]^{T}$, the RHS in (\ref{vgammaj_matnorm_ineq1}) can be expressed as 
a matrix norm $||\vecu^{r} - \vecu^{r+1}||_{\matG} = (\vecu^{r} - \vecu^{r+1})^{T} \matG (\vecu^{r+1} - \vecu^{*})$, where $\matG$ is given by 
\begin{equation*} \label{defn_G_mat}
 \matG =  
 \begin{bmatrix}
  \rho I_{n|B|} & 0 \\
       0 & \frac{1}{\rho}I_{N_{C}}
 \end{bmatrix}
 .
\end{equation*}
Using the identity:
\begin{eqnarray} \label{mat_norm_identity1}
&& \hspace{-1.2cm} 2(\vecu^{r} - \vecu^{r+1})^{T} \matG (\vecu^{r+1} - \vecu^{*}) = 
\nonumber \\
&& \hspace{-0.4cm}
||\vecu^{r} - \vecu^{*}||_{G}^{2}
- ||\vecu^{r+1} - \vecu^{*}||_{G}^{2}
- ||\vecu^{r} - \vecu^{r+1}||_{G}^{2},
\end{eqnarray}
the inequality in (\ref{vgammaj_matnorm_ineq1}) can be rewritten as
\begin{eqnarray} \label{vgammaj_matnorm_ineq2} 
&& \hspace{-1cm}
 m_{f} ||\vgamma_{\J}^{r+1} - \vgamma_{\J}^{*}||_{2}^{2} \leq 
\nonumber \\
&& \hspace{-0.8cm}
\frac{1}{2} 
\lb
||\vecu^{r} - \vecu^{*}||_{G}^{2}
- ||\vecu^{r+1} - \vecu^{*}||_{G}^{2}
- ||\vecu^{r} - \vecu^{r+1}||_{G}^{2} 
 \rb.
\end{eqnarray}
By discarding the non-positive terms in the LHS of (\ref{vgammaj_matnorm_ineq2}), we obtain the following upper bound on the primal optimality 
gap.
\begin{equation} \label{vgammaj_matnorm_ineq3} 
||\vgamma_{\J}^{r+1} - \vgamma_{\J}^{*}||_{2}^{2} 
 \; \leq \;
 \frac{1}{2 m_{f}} ||\vecu^{r} - \vecu^{*}||_{G}^{2}.
\end{equation}
In Appendix \ref{App:monotonic_convergence_proof}, we prove the monotonic convergence of $\vecu^{r}$ to $\vecu^{*}$. Thus, from the monotonic decay of the RHS in (\ref{vgammaj_matnorm_ineq3}), we have R-linear convergence of $\vgamma_{\J}^{r}$ to $\vgamma_{\J}^{*}$.

\subsection{Proof of monotonic convergence of $\vecu^{r}$ to $\vecu^{*}$}\label{App:monotonic_convergence_proof}
In order to prove monotonic convergence of $\vecu^{r}$ to $\vecu^{*}$, it is sufficient to show that there exists a $\delta > 0$ such that
\begin{equation} \label{vecu_monotonic_convergence}
 ||\vecu^{r+1} - \vecu^{*}||_{G}^{2} 
 \; \leq \;
 \frac{1}{1+ \delta} ||\vecu^{r} - \vecu^{*}||_{G}^{2}. 
\end{equation} 
By rearranging the terms in (\ref{vgammaj_matnorm_ineq2}), we have
\begin{eqnarray} \label{vgammaj_matnorm_ineq2a} 
||\vecu^{r+1} - \vecu^{*}||_{G}^{2} 
&\leq&
||\vecu^{r} - \vecu^{*}||_{G}^{2}
- ||\vecu^{r+1} - \vecu^{r}||_{G}^{2} 
\nonumber \\
&& 
- 2 m_{f} ||\vgamma_{\J}^{r+1} - \vgamma_{\J}^{*}||_{2}^{2}
.
\end{eqnarray}
By comparing the terms in (\ref{vecu_monotonic_convergence}) and (\ref{vgammaj_matnorm_ineq2a}), it is easy to see that if 
\begin{equation} \label{to_show_for_vecu_convergence_1}
2 m_{f} ||\vgamma_{\J}^{r+1} - \vgamma_{\J}^{*}||_{2}^{2} 
+ ||\vecu^{r+1} - \vecu^{r}||_{G}^{2} 
 \geq  
\delta ||\vecu^{r+1} - \vecu^{*}||_{G}^{2},	
\end{equation}
or equivalently,
\begin{eqnarray} \label{to_show_for_vecu_convergence_2}
&& \hspace{-0.5cm}
2 m_{f} ||\vgamma_{\J}^{r+1} - \vgamma_{\J}^{*}||_{2}^{2} 
+ \rho ||\matE_{2}(\vgamma_{\B}^{r+1} - \vgamma_{\B}^{r})||_{2}^{2}
\nonumber \\
&& \hspace{0.1cm}
+ \frac{1}{\rho} ||\vlambda^{r+1} - \vlambda^{r}||_{2}^{2}
\geq
\nonumber \\
&& \hspace{0.1cm}
\delta 
\lb
\rho ||\matE_{2}(\vgamma_{\B}^{r+1} - \vgamma_{\B}^{*})||_{2}^{2}
+ \frac{1}{\rho} ||\vlambda^{r+1} - \vlambda^{*}||_{2}^{2}
\rb,
\end{eqnarray}
then, we have monotonic convergence of $\vecu^{r}$ to $\vecu^{*}$. We now proceed to derive upper bounds for the RHS terms 
$||\matE_{2}(\vgamma_{\B}^{r+1} - \vgamma_{\B}^{*})||_{2}$ and $||\vlambda^{r+1} - \vlambda^{*}||_{2}$ in terms of the LHS terms. 
These upper bounds will be used in the sequel to establish the inequality in (\ref{to_show_for_vecu_convergence_2}).

\vspace{0.4cm}
\begin{itemize}
 \item \emph{An upper bound for $\rho||\matE_{2}(\vgamma_{\B}^{k+1} - \vgamma_{\B}^{*})||_{2}$}\\
 Note that for any two vectors $\veca$, $\vecb$ and a scalar $\mu > 1$
\begin{equation} \label{matrix_norm_ineq1} 
||\veca+\vecb||_{2}^{2} \; \geq \; (1-\mu)||\veca||_{2}^{2} + \left(1 - \frac{1}{\mu} \right)||\vecb||_{2}^{2}. 
\end{equation}
Applying inequality (\ref{matrix_norm_ineq1}) to (\ref{diff_term3}), we get the following upper bound.
\begin{eqnarray}
&& \rho||\matE_{2}(\vgamma_{\B}^{r+1} - \vgamma_{\B}^{*})||_{2}^{2} 
\leq 
\left(\frac{\mu}{\mu-1} \right) \frac{1}{\rho}||\vlambda^{r+1} - \vlambda^{r}||_{2}^{2} 
\nonumber \\
&& \hspace{1cm} 
+ 
\left(\mu \rho \sigma_{\text{max}}^{2}(\matE_{1}) \right) ||\vgamma_{\J}^{r+1} - \vgamma_{\J}^{*}||_{2}^{2}.    
\label{upper_bound_1}
\end{eqnarray}
Here, $\sigma_{\text{max}}(\matE_{1})$ is the largest singular value of $\matE_{1}$.

\vspace{0.4cm}
\item \emph{An upper bound for $\frac{1}{\rho}||\vlambda^{r+1} - \vlambda^{*}||_{2}$} \\
Similar application of inequality (\ref{matrix_norm_ineq1}) to (\ref{diff_term1}) results in an upper bound for $\frac{1}{\rho}||\vlambda^{r+1} - \vlambda^{*}||_{2}$
as shown below.
\begin{align}
& ||\matE_{1}^{T}(\vlambda^{r+1} - \vlambda^{*})||_{2}^{2} 
\; \leq \;
\nonumber \\
& \hspace{0.5cm}
\frac{\nu}{(\nu - 1)}||\nabla f(\vgamma_{\J}^{r+1})^{T} - \nabla f(\vgamma_{\J}^{*})^{T}||_{2}^{2}
\nonumber \\
& \hspace{0.5cm} 
\;+\; 
\nu||\rho \matE_{1}^{T}\matE_{2}(\vgamma_{\B}^{r} - \vgamma_{\B}^{r+1})||_{2}^{2}
\nonumber \\
\implies & \frac{1}{\rho}||\vlambda^{r+1} - \vlambda^{*}||_{2}^{2} 
\; \leq \;
\nonumber \\
& \hspace{0.5cm}
\frac{\nu}{\rho(\nu - 1)\sigma_{\text{min}}^{2}(\matE_{1})}||\nabla f(\vgamma_{\J}^{r+1})^{T} - \nabla f(\vgamma_{\J}^{*})^{T}||_{2}^{2} 
\nonumber \\
& \hspace{0.5cm} 
\; + \;  
\frac{\nu \rho\sigma_{\text{max}}^{2}(\matE_{1})}{\sigma_{\text{min}}^{2}(\matE_{1})}|| \matE_{2}(\vgamma_{\B}^{r} - \vgamma_{\B}^{r+1})||_{2}^{2}.  
\end{align} 
From Lipschitz continuity of $\nabla f$ (\ref{lipschitz_continuity_for_gradient}), we obtain the following modified upper bound.
\begin{align} \label{upper_bound_2}
& \frac{1}{\rho}||\vlambda^{r+1} - \vlambda^{*}||_{2}^{2} 
\; \leq \; 
\frac{\nu M_{f}^{2}}{\rho(\nu - 1)\sigma_{\text{min}}^{2}(\matE_{1})}
||\vgamma_{\J}^{r+1} - \vgamma_{\J}^{*}||_{2}^{2}  
\nonumber \\
& \hspace{1cm}
+ 
\frac{\nu \rho\sigma_{\text{max}}^{2}(\matE_{1})}{\sigma_{\text{min}}^{2}(\matE_{1})}|| \matE_{2}(\vgamma_{\B}^{r} - \vgamma_{\B}^{r+1})||_{2}^{2}.
\end{align}
Here, $\sigma_{\text{min}}(\matE_{1})$ denotes the smallest singular value of $\matE_{1}$ and $\nu$ is a positive scalar greater than unity.
\end{itemize}

By summing the upper bounds in \eqref{upper_bound_1} and \eqref{upper_bound_2}, we get
\begin{align}
&  
 \rho ||\matE_{2}(\vgamma_{\B}^{r+1} - \vgamma_{\B}^{*})||_{2}^{2}
+ \frac{1}{\rho}||\vlambda^{r+1} - \vlambda^{*}||_{2}^{2} 
\; \leq \;
\nonumber \\
& \hspace{1cm} 
\displaystyle
\frac{1}{\delta}
\left (
2 m_{f}||\vgamma_{\J}^{r+1} - \vgamma_{\J}^{*}||_{2}^{2}
+
\rho|| \matE_{2}(\vgamma_{\B}^{r} - \vgamma_{\B}^{r+1})||_{2}^{2} 
\right.
\nonumber \\
& \hspace{1cm} 
+ 
\left. 
\frac{1}{\rho}||\vlambda^{r+1} - \vlambda^{r}||_{2}^{2} \right )
\label{delta_ineq} 
\end{align} 
where
\begin{equation} \label{defn_delta1}
\delta 
\! \triangleq \!
\lb
\underset{\mu, \nu \geq 1}{\text{max}} \lb \displaystyle
\frac{\frac{\nu M_{f}^{2}}{\rho(\nu - 1)\sigma_{\text{min}}^{2}(\matE_{1})} \!+\!
\mu \rho \sigma_{\text{max}}^{2}(\matE_{1})}{2 m_{f}} 
,
\nu \kappa
,
\frac{\mu}{\mu -1}
\rb
\rb ^{-1}.
\end{equation}
Thus, for $\delta$ as defined above, the inequality (\ref{to_show_for_vecu_convergence_2}) holds and consequently the inequality
(\ref{vecu_monotonic_convergence}) also holds, thereby establishing the Q-linear convergence of the sequence $\vecu^{k}$ to $\vecu^{*}$.

\subsection{Proof of Theorem \ref{theorem_optimal_rho_n_delta}} \label{App:theorem_optimal_rho_n_delta_proof}

Let $\delta_{\text{opt}}$ denote the maximum value of $\delta$ for any $\rho > 0$. Then, we can write
\begin{align} \label{defn_delta_opt_compact} 
\delta_{\text{opt}} &=
\displaystyle
\underset{\rho > 0}{\text{max}} \lb
\underset{\mu, \nu \geq 1}{\text{max}} \lb
\text{min} \lb
f_{1}(\mu, \nu, \rho), f_{2}(\nu), f_{3}(\mu)
\rb \rb \rb 
\nonumber \\
&= \;
\underset{\mu, \nu \geq 1}{\text{max}} \lb
\displaystyle
\underset{\rho > 0}{\text{max}} \lb
\text{min} \lb
f_{1}(\mu, \nu, \rho), f_{2}(\nu), f_{3}(\mu)
\rb \rb \rb 
\end{align}
where the scalar functions $f_{1}$, $f_{2}$ and $f_{3}$ represent the three terms inside the minimum operator in  (\ref{defn_delta}). 
The following two Lemmas summarize the optimization of $\delta$ in (\ref{defn_delta_opt_compact}).


\begin{lemma} \label{lem_alternate_delta_opt_algo}
$\delta_{\text{opt}} = \underset{\mu, \nu \geq 1}{\text{max}} \lc \text{min} \lb \bar{f}_{1}(\mu, \nu), f_{2}(\nu), f_{3}(\mu) \rb \rc $
where, $\bar{f}_{1}(\mu, \nu) \triangleq \underset{\rho > 0}{\text{max}} \; f_{1}(\mu, \nu, \rho) $.
\end{lemma}
\begin{proof}
See Appendix \ref{App:lem_alternate_delta_opt_algo_proof}.
\end{proof}

\begin{lemma} \label{lemma_equal_is_optimal}
There exists a unique $(\mu, \nu) = \lb\mu^{*}, \nu^{*}\rb$ which simultaneously satisfies 
\begin{enumerate}
 \item $ \bar{f}_{1} = f_{2} = f_{3}$
 \item $\mu \ge 1, \nu \ge 1$.
\end{enumerate}
Further, such a $(\mu^*, \nu^*)$ maximizes $g(\mu, \nu) = \text{min}\; \lb \bar{f}_{1}(\mu, \nu), f_{2}(\nu), f_{3}(\mu) \rb$ over $\mu, \nu \ge 1$.   
\end{lemma}
\begin{proof}
See Appendix \ref{App:lem_equal_is_optimal_proof}.
\end{proof}
The scalar function $f_{1}$ in Lemma \ref{lem_alternate_delta_opt_algo} is maximized 
at $\rho = \dfrac{M_f}{\sigma_{\text{max}} \sigma_{\text{min}}} \sqrt{\dfrac{\nu}{\mu(\nu-1)}}$ to 
give $\bar{f_{1}}= \displaystyle\frac{M_{f}}{\sigma_{\text{min}}\sigma_{\text{max}}}\sqrt{\frac{\nu}{\mu(\nu-1)}}$. 
Further, by solving for the unique tuple $(\mu^*, \nu^*)$ which satisfies the two optimality conditions specified 
in Lemma \ref{lemma_equal_is_optimal}, the optimal augmented Lagrangian parameter $\rho$ and corresponding 
optimal $\delta$ can be shown to be equal to the $\rho_{\text{opt}}$ and $\delta_{\text{opt}}$ as defined 
in Theorem \ref{theorem_optimal_rho_n_delta}.

\subsection{Proof of Lemma \ref{lem_alternate_delta_opt_algo}} \label{App:lem_alternate_delta_opt_algo_proof}
Let $\bar{\rho} \triangleq \underset{\rho > 0}{\text{arg max }} f_{1}$. Then, by restricting the feasible set in (\ref{defn_delta_opt_compact}), we have,
\begin{eqnarray}
 \delta_{\text{opt}} &\geq &  \underset{\mu, \nu \geq 1}{\text{max}} \ls \underset{\rho = \rho_{\mu, \nu}}{\text{max}} \lc \text{min} \lb f_{1}(\mu, \nu, \rho), f_{2}(\nu), f_{3}(\mu) \rb \rc \rs 
 \nonumber
 \\
 &=&  \underset{\mu, \nu \geq 1}{\text{max}} \lc \text{min} \lb \tilde{f}_{1}(\mu, \nu), f_{2}(\nu), f_{3}(\mu) \rb \rc. 
 \label{delta_opt_lb}
\end{eqnarray}
On the other hand, from (\ref{defn_delta_opt_compact}) and using $\tilde{f}_{1} \ge f_{1}$, 
we have,
\begin{eqnarray}
 \delta_{\text{opt}} &=& 
 \underset{\mu, \nu \geq 1}{\text{max}} \ls \underset{\rho > 0}{\text{max}} \lc \text{min} \lb f_{1}(\mu, \nu, \rho), f_{2}(\nu), f_{3}(\mu) \rb \rc \rs   
 \nonumber
 \\
 &\le&  \underset{\mu, \nu \geq 1}{\text{max}} \lc \text{min} \lb \tilde{f}_{1}(\mu, \nu), f_{2}(\nu), f_{3}(\mu) \rb \rc.
  \label{delta_opt_ub}
\end{eqnarray}
Combining (\ref{delta_opt_lb}) and (\ref{delta_opt_ub}) establishes Lemma \ref{lem_alternate_delta_opt_algo}.

\subsection{Proof of Lemma \ref{lemma_equal_is_optimal}} \label{App:lem_equal_is_optimal_proof}
In order to prove the Lemma, we claim the following.
\begin{enumerate}[label=\alph*)]
 \item For any $\epsilon > 0$, there exist positive constants $B_{\mu}$ and $B_{\nu}$ such that $g(\mu, \nu) \le \epsilon$ when either $\mu \ge B_{\mu}$ or $\nu \ge B_{\nu}$ holds.
 \item Any points $(\mu, \nu)$ which satisfies condition $2$ but does not satisfy condition $1$ cannot be a local maximum of $g$.
\end{enumerate}
Note that claim (a) holds trivially for $B_{\mu} = \frac{m_{f}^{2}}{\kappa \M_{f}^{2} \epsilon^{2}}$ and $B_{\nu} = \frac{1}{\kappa \epsilon}$. 
In order to verify claim (b), let us consider a point $(\mu_{0}, \nu_{0})$ which satisfies condition $2$, but not condition $1$. Then, we need to consider three cases.
\begin{itemize}
 \item \emph{Case-I:} $\tilde{f}_{1}$, $f_{2}$ and $f_{3}$ are distinct at $(\mu_{0}, \nu_{0})$. Without loss of generality, let $g = \tilde{f}_{1}$ at $(\mu_{0}, \nu_{0})$.
 Then, from the continuity of $\tilde{f_{1}}, f_{2}, f_{3}$, there exists an $\epsilon \; (> 0)$ ball $B_{\epsilon}$, centered at $(\mu_{0}, \nu_{0})$ and 
 with radius $\epsilon$ inside which $g = \tilde{f_{1}}$ holds. Since, inside $B_{\epsilon}$, $g$ is strictly monotonic with respect to $\mu$ and $\nu$,
 there exists $(\mu, \nu) \in \B_{\epsilon}$ such that $g(\mu, \nu) > g(\mu_{0}, \nu_{0})$. Hence, $(\mu_{0}, \nu_{0})$ is not a local maximum. 
 \item \emph{Case-II:} At $(\mu_{0}, \nu_{0})$, any two of $\tilde{f}_{1}$, $f_{2}$ and $f_{3}$ are equal and strictly greater than the remaining one. 
 The same arguments as Case-I apply here as well.
 \item \emph{Case-III:}At $(\mu_{0}, \nu_{0})$, any two of $\tilde{f}_{1}$, $f_{2}$ and $f_{3}$ are equal and strictly less than the remaining one.
 WLOG, let $\tilde{f}_{1} = f_{2} < f_{3}$. Let $\mathcal{C}(\mu, \nu)$ denote the continuous curve in $(\mu, \nu)$ plane whose each point satisfies $\tilde{f}_{1} = f_{2}$.
 Clearly, $(\mu_{0}, \nu_{0})$ also lies on the curve $\mathcal{C}$. Moreover, there are 
are an uncountably infinite number of points of $\mathcal{C}$ inside $B_{\epsilon}$, with $B_{\epsilon}$ defined as in Case-I. Due to the monotonicity of $g$ along $\mathcal{C}$, there exists $(\mu, \nu) \in \B_{\epsilon}$ such that $g(\mu, \nu) > g(\mu_{0}, \nu_{0})$. 
 Hence, $(\mu_{0}, \nu_{0})$ is not a local maximum. 
\end{itemize}
From claim (a) and the fact that at the boundary points $(\mu =1 \; \text{or} \; \nu =1)$, the objective $g$ evaluates to zero, 
we may restrict our search for the global maximizer of $g$ to set $\mathcal{D} = \lc (\mu, \nu) \; \middle| \; 1 \le  \mu \le B_{\mu}, 1 \le  \nu \le B_{\nu} \rc$.  
Then, from claim (b), uniqueness of $(\mu^{*}, \nu^{*}) \in D$ and Weierstrass theorem, it follows that $(\mu^{*}, \nu^{*})$ is indeed
the unique global maximizer of the continuous function $g$. Thus, the proof is complete.


\bibliographystyle{IEEEtran}
\bibliography{IEEEabrv,bibJournalList,CBDSBL_JLversion}

\begin{thebibliography}{10}
\providecommand{\url}[1]{#1}
\csname url@samestyle\endcsname
\providecommand{\newblock}{\relax}
\providecommand{\bibinfo}[2]{#2}
\providecommand{\BIBentrySTDinterwordspacing}{\spaceskip=0pt\relax}
\providecommand{\BIBentryALTinterwordstretchfactor}{4}
\providecommand{\BIBentryALTinterwordspacing}{\spaceskip=\fontdimen2\font plus
\BIBentryALTinterwordstretchfactor\fontdimen3\font minus
  \fontdimen4\font\relax}
\providecommand{\BIBforeignlanguage}[2]{{%
\expandafter\ifx\csname l@#1\endcsname\relax
\typeout{** WARNING: IEEEtran.bst: No hyphenation pattern has been}%
\typeout{** loaded for the language `#1'. Using the pattern for}%
\typeout{** the default language instead.}%
\else
\language=\csname l@#1\endcsname
\fi
#2}}
\providecommand{\BIBdecl}{\relax}
\BIBdecl

\bibitem{Khanna14CBDSBL}
S.~Khanna and C.~Murthy, ``Decentralized {B}ayesian learning of jointly sparse
  signals,'' in \emph{Proc. GLOBECOM}, Dec 2014, pp. 3103--3108.

\bibitem{Duarte05DCS}
M.~Duarte, S.~Sarvotham, D.~Baron, M.~Wakin, and R.~Baraniuk, ``Distributed
  compressed sensing of jointly sparse signals,'' in \emph{Proc.\ Asilomar
  Conf.\ on Signals, Syst., and Comput.}, Oct 2005, pp. 1537--1541.

\bibitem{Makhzani13DCSAMP}
A.~Makhzani and S.~Valaee, ``Distributed spectrum sensing in cognitive radios
  via graphical models,'' in \emph{Proc.\ {CAMSAP}}, 2013, pp. 376--379.

\bibitem{Tian11DistSpectrumSensing}
Z.~Fanzi, C.~Li, and Z.~Tian, ``Distributed compressive spectrum sensing in
  cooperative multihop cognitive networks,'' \emph{{IEEE} J. Sel. Topics Signal
  Process.}, vol.~5, no.~1, pp. 37--48, 2011.

\bibitem{Ling_13_jsm_lqnorm}
Q.~Ling, Z.~Wen, and W.~Yin, ``Decentralized jointly sparse optimization by
  reweighted l-q minimization,'' \emph{{IEEE} Trans. Signal Process.}, vol.~61,
  no.~5, pp. 1165--1170, 2013.

\bibitem{Nasrabadi11MTMV}
N.~Nguyen, N.~Nasrabadi, and T.~Tran, ``Robust multi-sensor classification via
  joint sparse representation,'' in \emph{Proc. 14th Int. Conf. Inform.
  Fusion}, July 2011, pp. 1--8.

\bibitem{Shihao09MultiTaskCS}
S.~Ji, D.~Dunson, and L.~Carin, ``Multitask compressive sensing,'' \emph{{IEEE}
  Trans. Signal Process.}, vol.~57, no.~1, pp. 92--106, Jan 2009.

\bibitem{Ranjitha15SBLChEst}
R.~Prasad, C.~Murthy, and B.~Rao, ``Joint channel estimation and data detection
  in {MIMO}-{OFDM} systems: A sparse {B}ayesian learning approach,''
  \emph{{IEEE} Trans. Signal Process.}, vol.~63, no.~20, pp. 5369--5382, Oct
  2015.

\bibitem{Masood15SparseMIMOChEst}
M.~Masood, L.~Afify, and T.~Al-Naffouri, ``Efficient coordinated recovery of
  sparse channels in massive {MIMO},'' \emph{{IEEE} Trans. Signal Process.},
  vol.~63, no.~1, pp. 104--118, 2015.

\bibitem{Barbotin12MIMOJSM}
Y.~Barbotin, A.~Hormati, S.~Rangan, and M.~Vetterli, ``Estimation of sparse
  {MIMO} channels with common support,'' \emph{{IEEE} Trans. Commun.}, vol.~60,
  no.~12, pp. 3705--3716, 2012.

\bibitem{Cotter_05_mmv}
S.~Cotter, B.~Rao, K.~Engan, and K.~Kreutz-Delgado, ``Sparse solutions to
  linear inverse problems with multiple measurement vectors,'' \emph{{IEEE}
  Trans. Signal Process.}, vol.~53, no.~7, pp. 2477--2488, 2005.

\bibitem{Duarte09DCS_main}
\BIBentryALTinterwordspacing
D.~Baron, M.~F. Duarte, M.~B. Wakin, S.~Sarvotham, and R.~G. Baraniuk,
  ``Distributed compressive sensing,'' \emph{CoRR}, vol. abs/0901.3403, 2009.
  [Online]. Available: \url{http://arxiv.org/abs/0901.3403}
\BIBentrySTDinterwordspacing

\bibitem{Lu11AdmmMmv}
H.~Lu, X.~Long, and J.~Lv, ``A fast algorithm for recovery of jointly sparse
  vectors based on the alternating direction methods,'' \emph{J. Mach. Learn.
  Res}, pp. 461--469, 2011.

\bibitem{Wipf_07_msbl}
D.~Wipf and B.~Rao, ``An empirical bayesian strategy for solving the
  simultaneous sparse approximation problem,'' \emph{{IEEE} Trans. Signal
  Process.}, vol.~55, no.~7, pp. 3704--3716, 2007.

\bibitem{Wipf_04_sbl}
------, ``Sparse bayesian learning for basis selection,'' \emph{{IEEE} Trans.
  Signal Process.}, vol.~52, no.~8, pp. 2153--2164, 2004.

\bibitem{Ziniel13AMPMMV}
J.~Ziniel and P.~Schniter, ``Efficient high-dimensional inference in the
  multiple measurement vector problem,'' \emph{{IEEE} Trans. Signal Process.},
  vol.~61, no.~2, pp. 340--354, Jan 2013.

\bibitem{Rakotomamonjy_survey}
A.~Rakotomamonjy, ``Review: Surveying and comparing simultaneous sparse
  approximation (or group-lasso) algorithms,'' \emph{Signal Processing},
  vol.~91, no.~7, pp. 1505--1526, Jul. 2011.

\bibitem{Wimalajeewa13DCOMP}
T.~Wimalajeewa and P.~Varshney, ``Cooperative sparsity pattern recovery in
  distributed networks via distributed-{OMP},'' in \emph{Proc.\ {ICASSP}}, May
  2013, pp. 5288--5292.

\bibitem{Varshney14DCSP}
G.~Li, T.~Wimalajeewa, and P.~Varshney, ``Decentralized subspace pursuit for
  joint sparsity pattern recovery,'' in \emph{Proc.\ {ICASSP}}, May 2014, pp.
  3365--3369.

\bibitem{LingT11DecentralizedSuppportDetect}
Q.~Ling and Z.~Tian, ``Decentralized support detection of multiple measurement
  vectors with joint sparsity,'' in \emph{Proc.\ {ICASSP}}, May 2011, pp.
  2996--2999.

\bibitem{Bach12SparsityInducingPenalties}
\BIBentryALTinterwordspacing
F.~Bach, R.~Jenatton, J.~Mairal, and G.~Obozinski, ``Optimization with
  sparsity-inducing penalties,'' \emph{Foundations and Trends in Machine
  Learning}, vol.~4, no.~1, pp. 1--106, Jan. 2012. [Online]. Available:
  \url{http://dx.doi.org/10.1561/2200000015}
\BIBentrySTDinterwordspacing

\bibitem{Yang10TurboBCS}
D.~Yang, H.~Li, and G.~Peterson, ``Space-time turbo bayesian compressed sensing
  for {UWB} systems,'' in \emph{Proc.\ {ICC}}, May 2010, pp. 1--6.

\bibitem{Giannakis_08_NoisyLinks}
I.~Schizas, A.~Ribeiro, and G.~Giannakis, ``Consensus in ad hoc {WSN}s with
  noisy links; part {I}: Distributed estimation of deterministic signals,''
  \emph{{IEEE} Trans. Signal Process.}, vol.~56, no.~1, pp. 350--364, 2008.

\bibitem{Wipf03perspectives}
\BIBentryALTinterwordspacing
J.~Palmer, B.~D. Rao, and D.~P. Wipf, ``Perspectives on sparse {B}ayesian
  learning,'' in \emph{Advances in Neural Information Processing Systems},
  2004, pp. 249--256. [Online]. Available:
  \url{http://papers.nips.cc/paper/2393-perspectives-on-sparse-bayesian-learni%
ng.pdf}
\BIBentrySTDinterwordspacing

\bibitem{Neal98IncrementalEM}
R.~Neal and G.~E. Hinton, ``A view of the {EM} algorithm that justifies
  incremental, sparse, and other variants,'' in \emph{Learning in Graphical
  Models}.\hskip 1em plus 0.5em minus 0.4em\relax Kluwer Academic Publishers,
  1998, pp. 355--368.

\bibitem{CoverThomasBook}
T.~M. Cover and J.~A. Thomas, \emph{Elements of Information Theory (Wiley
  Series in Telecommunications and Signal Processing)}.\hskip 1em plus 0.5em
  minus 0.4em\relax Wiley-Interscience, 2006.

\bibitem{Dempster77EM}
A.~P. Dempster, N.~M. Laird, and D.~B. Rubin, ``Maximum likelihood from
  incomplete data via the {EM} algorithm,'' \emph{Journal of the Royal
  Statistical Society, series B}, vol.~39, no.~1, pp. 1--38, 1977.

\bibitem{Giannakis_08_CBDEM}
P.~Forero, A.~Cano, and G.~Giannakis, ``Consensus-based distributed
  expectation-maximization algorithm for density estimation and classification
  using wireless sensor networks,'' in \emph{Proc.\ {ICASSP}}, Mar 2008, pp.
  1989--1992.

\bibitem{Parikh_11_ADMM}
S.~Boyd, N.~Parikh, E.~Chu, B.~Peleato, and J.~Eckstein, ``Distributed
  optimization and statistical learning via the alternating direction method of
  multipliers,'' \emph{Foundations and Trends in Machine Learning}, vol.~3,
  no.~1, pp. 1--122, 2010.

\bibitem{HaoZhu09CAMoM}
H.~Zhu, G.~Giannakis, and A.~Cano, ``Distributed in-network channel decoding,''
  \emph{Signal Processing, IEEE Transactions on}, vol.~57, no.~10, pp.
  3970--3983, Oct 2009.

\bibitem{Mota_13_ADMM}
J.~Mota, J.~Xavier, P.~Aguiar, and M.~Puschel, ``{D-ADMM}: A
  communication-efficient distributed algorithm for separable optimization,''
  \emph{{IEEE} Trans. Signal Process.}, vol.~61, no.~10, pp. 2718--2723, 2013.

\bibitem{WeiShi15EXTRA}
\BIBentryALTinterwordspacing
W.~Shi, Q.~Ling, G.~Wu, and W.~Yin, ``{EXTRA}: An exact first-order algorithm
  for decentralized consensus optimization,'' \emph{SIAM Journal on
  Optimization}, vol.~25, no.~2, pp. 944--966, 2015. [Online]. Available:
  \url{http://dx.doi.org/10.1137/14096668X}
\BIBentrySTDinterwordspacing

\bibitem{Matamoros151bitADMMJSM1}
J.~Matamoros, S.~Fosson, E.~Magli, and C.~Anton-Haro, ``Distributed {ADMM} for
  in-network reconstruction of sparse signals with innovations,'' \emph{{IEEE}
  Trans. Signal Inf. Process. Netw.}, vol.~1, no.~4, pp. 225--234, Dec 2015.

\bibitem{Erseghe15MLthruADMM}
T.~Erseghe, ``A distributed and maximum-likelihood sensor network localization
  algorithm based upon a nonconvex problem formulation,'' \emph{{IEEE} Trans.
  Signal Inf. Process. Netw.}, vol.~1, no.~4, pp. 247--258, Dec 2015.

\bibitem{WeiDengWotaoYinADMM12}
W.~Deng and W.~Yin, ``On the global and linear convergence of the generalized
  alternating direction method of multipliers,'' \emph{Rice University CAAM
  Technical Report TR12-14}, 2012.

\bibitem{WataoYin13ADMMConvergence}
W.~Shi, Q.~Ling, K.~Yuan, G.~Wu, and W.~Yin, ``On the linear convergence of the
  {ADMM} in decentralized consensus optimization,'' \emph{{IEEE} Trans. Signal
  Process.}, vol.~62, no.~7, pp. 1750--1761, Apr. 2014.

\bibitem{Erseghe11AvgConsensusADMM}
T.~Erseghe, D.~Zennaro, E.~Dall'Anese, and L.~Vangelista, ``Fast consensus by
  the alternating direction multipliers method,'' \emph{{IEEE} Trans. Signal
  Process.}, vol.~59, no.~11, pp. 5523--5537, Nov 2011.

\bibitem{Khanna15CBDSBLarXiV}
\BIBentryALTinterwordspacing
S.~Khanna and C.~R. Murthy, ``Consensus based decentralized sparse bayesian
  learning for joint sparse signal recovery,'' \emph{CoRR}, vol.
  abs/1507.02387, 2015. [Online]. Available:
  \url{http://arxiv.org/abs/1507.02387}
\BIBentrySTDinterwordspacing

\bibitem{Jakovetic14FastDistributedGradientMethods}
D.~Jakovetic, J.~Xavier, and J.~Moura, ``Fast distributed gradient methods,''
  \emph{IEEE Trans.\ Automat.\ Control}, vol.~59, no.~5, pp. 1131--1146, May
  2014.

\bibitem{ChenAndOzdaglar12FastDistProxGradMethod}
A.~Chen and A.~Ozdaglar, ``A fast distributed proximal-gradient method,'' in
  \emph{Proc.\ Allerton Conf.\ on Commun., Control and Comput.}, Oct 2012, pp.
  601--608.

\end{thebibliography}

\end{document}